\documentclass[11pt]{article}

\usepackage{amsmath,amssymb,graphicx,color,amsthm, mathrsfs}
\usepackage[linesnumbered,ruled]{algorithm2e}
\usepackage{tikz}
\usetikzlibrary{shapes.geometric, arrows}
\usepackage{algorithmic}
\usepackage[english]{babel}
\usepackage[utf8]{inputenc}
\usepackage{subcaption,afterpage} 
\usepackage{mathtools}
\usepackage{enumitem}
\usepackage{amsfonts}
\usepackage[colorinlistoftodos]{todonotes}
\usepackage[round]{natbib}
\usepackage{hyperref}
\newcommand{\excise}[1]{}

\newcommand\RR{\mathbb{R}}

\newcommand\ZZ{\mathbb{Z}}
\newcommand\EE{\mathbb{E}}

\newcommand{\tr}{\operatorname{tr}}

\renewcommand\>{\rangle}

\newcommand{\bfx}{{X}}

\newcommand{\bfv}{{v}}
\newcommand{\bfc}{{c}}
\newcommand{\cnh}{\widehat{c}}

\newcommand{\vnh}{\widehat{V}}
\newcommand{\rnh}{\widehat{r}}
\newcommand{\Id}{\mathrm{I}}

\newtheorem{theorem}{Theorem}
\newtheorem{definition}{Definition}
\newtheorem{lemma}{Lemma}
\newtheorem{example}{Example}
\newtheorem{proposition}{Proposition}
\newtheorem{corollary}{Corollary}
\newtheorem{remark}{Remark}

\DeclareMathOperator\var{var}
\DeclareMathOperator\diag{diag}

\DeclareMathOperator\argmin{argmin}
\DeclareMathOperator\argmax{argmax}
\DeclareMathOperator\supp{Supp}
\newcommand{\cov}{\mathrm{Cov}}

\newcommand{\RNum}[1]{\uppercase\expandafter{\romannumeral #1\relax}}

\begin{document}

	\title{\mbox{}\\[-11ex]Efficient Manifold Approximation with Spherelets}
	\vspace{-5ex}
	\author{\\[1ex]Didong Li$^{1,2}$, Minerva Mukhopadhyay$^3$ and David B Dunson$^{4}$ \\ 
	{\em Department of Computer Science$^1$, Princeton University}\\
	{\em Department of Biostatistics$^2$, University of California, Los Angeles}\\	
	{\em Department of Mathematics and Statistics$^2$, Indian Institute of Technology Kanpur}\\
	{\em Department of Statistical Science$^4$, Duke University}}
	\date{\vspace{-5ex}}
	
	\maketitle
	
In statistical dimensionality reduction, it is common to rely on the assumption that high dimensional data tend to concentrate near a lower dimensional manifold.  There is a rich literature on approximating the unknown manifold, and on exploiting such approximations in clustering, data compression, and prediction.  Most of the literature relies on linear or locally linear approximations.
In this article, we propose a simple and general alternative, which instead uses spheres, an approach we refer to as spherelets. We develop spherical principal components analysis (SPCA), and provide theory on the convergence rate for global and local SPCA, while showing that spherelets can provide lower covering numbers and MSEs for many manifolds.  Results relative to state-of-the-art competitors show gains in ability to accurately approximate manifolds with fewer components. Unlike most competitors, which simply output lower-dimensional features, our approach projects data onto the estimated manifold to produce fitted values that can be used for model assessment and cross validation. The methods are illustrated with applications to multiple data sets.

	
	\section{Introduction}\label{sec:intro}
Dimensionality reduction is a key step in statistical analyses of high-dimensional data.  
If one is willing to assume that data are concentrated close to a lower-dimensional hyperplane, then Principal Components Analysis (PCA) and its variants are natural.  The focus of this article is on relaxing the assumption of a lower-dimensional linear structure to allow the true latent structure to be curved; this motivates appropriate generalizations of PCA. 

This problem relates to {\em manifold learning}, which is based on the assumption that higher-dimensional data $X_i \in \RR^D$ are often concentrated near a $d$-dimensional Riemannian manifold $M$ with $d\ll D$.  The unknown manifold $M$ is approximately flat in small local neighborhoods but has non-zero curvature.  {As a simple example to provide motivation, in Figure 1 we plot economics data collected by the U.S. Bureau of Economic Analysis,
	and retrieved from the Federal Reserve Bank of St. Louis for economic
	analysis (available from the ggplot2 R package and \url{https://fred.stlouisfed.org} \citep{economic2020UEMPMED,economic2020PCE,economic2020PSAVERT,economic2020POP,economic2020UNEMPLOY})} containing 576 multivariate observations of (1) personal consumption expenditures in billions of dollars (consume), (2) total population in thousands (pop size), (3) personal saving rate (saving), (4) median duration of unemployment in weeks (dur unemploy), and (5) number of unemployed in thousands (num unemploy).  The pairwise plots are suggestive that the data may concentrate near a one-dimensional curve.

\begin{figure}[!h]
	\centering
	\includegraphics[width=\textwidth,height=200pt]{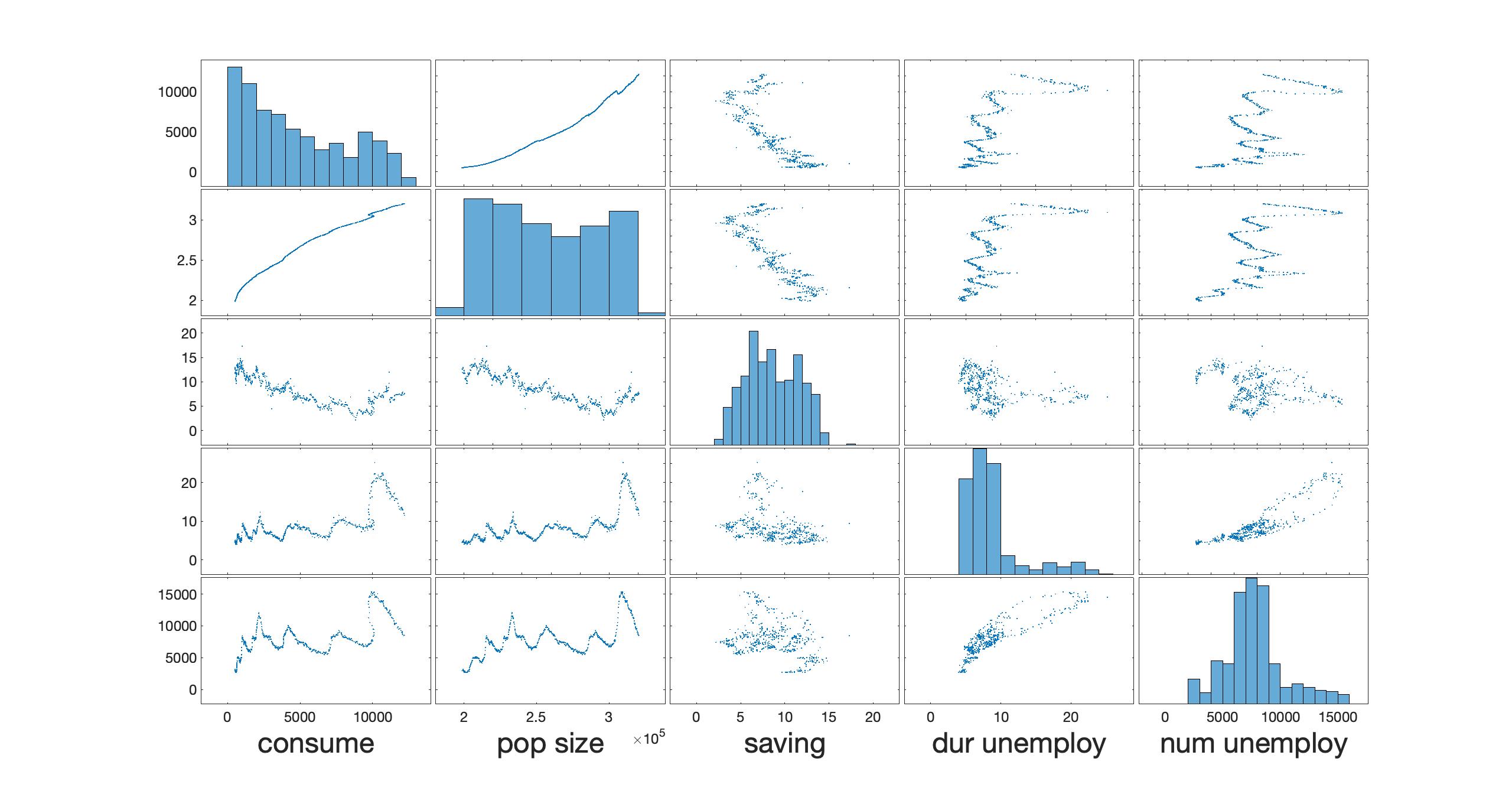} 
	\caption{Economics data showing pairwise scatterplots for each of five variables along with histograms for each variable.}
	\label{fig:economics}
\end{figure}

Much of the focus of this paper is on devising a simple method that can parsimoniously fit the data from Figure 1.  This rules out most of the existing manifold learning methods, ranging from Isomap \citep{isomap2000} to Diffusion Maps \citep{DM2006}. Such approaches focus on exploiting the manifold structure to replace the original data with lower dimensional features, but do not provide fitted values in the original data space.  For the economics data assuming a manifold dimension of $d=1$, such methods would replace the original $5$ features with a single feature.  This may provide a useful one-dimensional summary, but much of the interpretability is lost.

An alternative general strategy that allows one to obtain fitted values of data lying close to a nonlinear manifold is to rely on local approximations.  In particular, if we break up the original data domain into local regions, then within each region we could define a separate local approximation to the manifold.
When applying such approaches, by far the most common strategy is to rely on locally linear approximations, fitting a separate hyperplane within each region.  Such an approach is intuitively reasonable from a geometric perspective, as it is well known that Riemannian manifolds can be approximated via local tangent planes.  Some examples of this strategy include local PCA \citep{weingessel2000local,arias2017spectral} and geometric multiresolution analysis \citep{GMRA2011,GMRA2012,GMRA2016}.

As illustration, consider applying such an approach to the economics data of Figure \ref{fig:economics}. With some risk of misinterpretation in examining pairwise dependence plots, it appears that the one-dimensional manifold the data are concentrated near is very wiggly.  This implies that, in using local linear approximations, we will need to rely on a large number of well chosen neighborhoods to obtain adequate performance.  However, as the number of neighborhoods increases, the amount of data within each neighborhood decreases, and statistical uncertainty in estimating the local linear parameters increases.  The resulting estimator will tend to be noisy, with a tendency towards over-fitting.  The second column of Figure \ref{fig:economics_1425} shows the results of applying a state-of-the-art local linear approximation to the economics data, corresponding to an MSE of $2.5\times 10^{5}$.  The figure shows that the fitted curves are overly jagged.

\begin{figure}[!h]
	\centering
	\includegraphics[width=\textwidth]{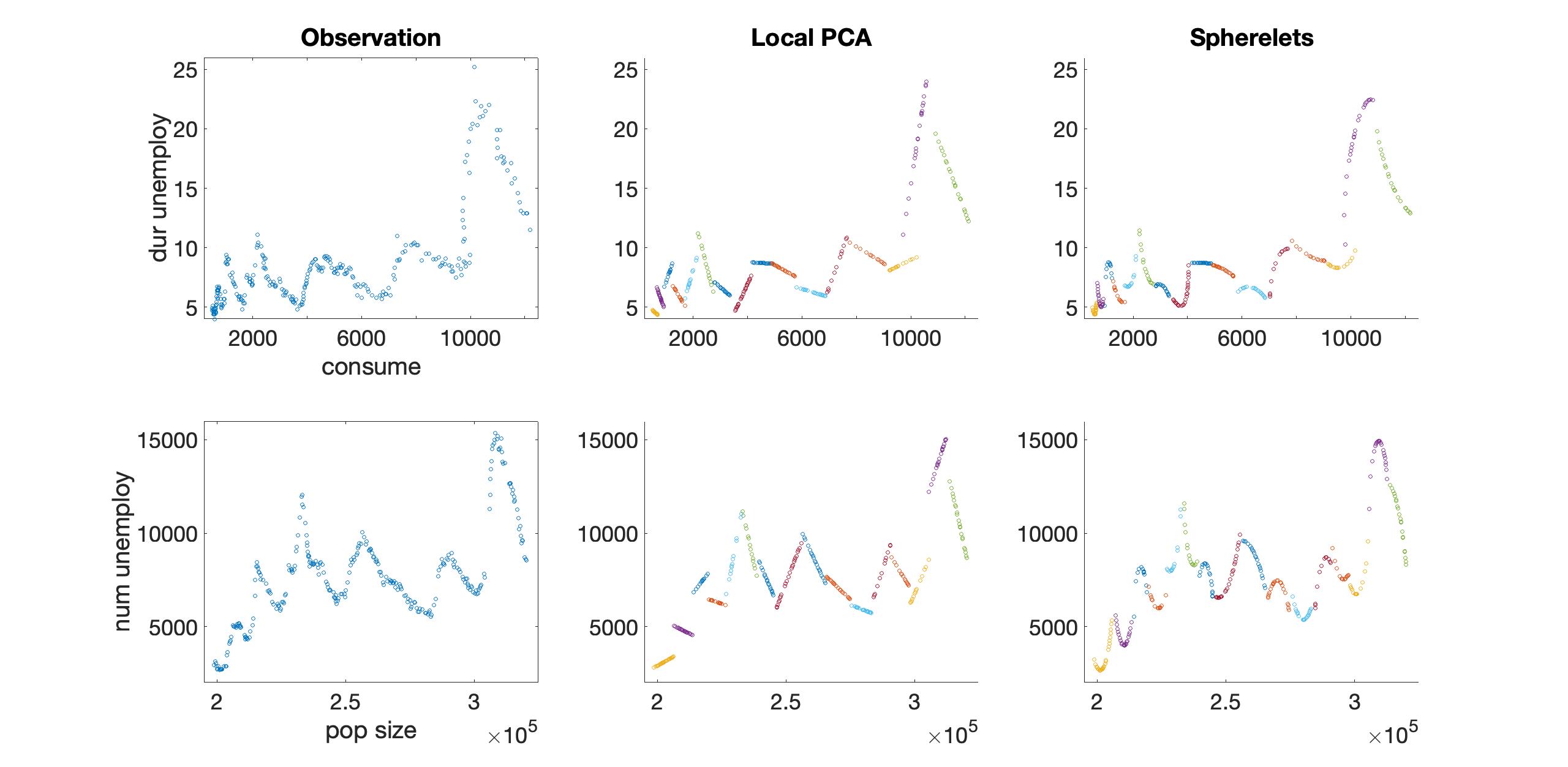} 
	\caption{Illustration of results based on fitting local PCA and local SPCA (spherelets) to the Economics data from Figure 1.  The number of components was chosen by cross validation to avoid over-fitting.
	}
	\label{fig:economics_1425}
\end{figure}

For one-dimensional manifolds there are various possibilities we could consider to improve the performance. If we are willing to assume that there is a single connected manifold, as appears to be the case for the economics data, then we could modify the linear fitting algorithm to include a constraint that the line segments in adjacent neighborhoods are connected.  Such a strategy is difficult to extend to higher dimensions, and would perform poorly if the data were instead from multiple disconnected manifolds.  Alternatively, one could apply a method designed for fitting a curve through data in $\RR^D$, such as principal curves \citep{hastie1989principal}. However, we applied principal curves to the economics data and found the curve is dramatically over-smoothed, providing very poor fit with MSE $4.0\times10^8$; hence, we do not include the plots in Figure \ref{fig:economics_1425} (see Figure \ref{fig:economics_all} in Appendix).  Principal curves cannot be applied beyond the $d=1$ case, and even for $d=1$ results are often non-stable and counter-intuitive.


Hence, we are motivated to return to the local approximation strategy, but to improve upon local PCA and alternative local linear approaches.
Much of the literature on nonlinear extensions of PCA is not relevant to our problem, because it involves replacing the original dataset with higher dimensional modifications based on applying kernels, polynomials, etc.  It would be conceptually possible to fit a quadratic or higher order surface within each neighborhood, but then we are substituting the problem of too many neighborhoods for one of too many parameters per neighborhood.  The resulting model fitting will be highly complex, and it is not clear that the bias-variance tradeoff will warrant such complexity.  

Instead, we propose a simple and efficient alternative to PCA, which uses spheres instead of hyperplanes to locally approximate manifolds. Our approach  relies on a spherical generalization of PCA, deemed SPCA. SPCA has a simple analytic form, and can be used as a generalization of PCA to incorporate curvature.  We refer to any algorithm that uses spheres for local approximation as {\em spherelets}. Spheres provide an excellent basis for locally approximating non-linear manifolds $M$ having  positive, zero, or negative Gaussian curvature that may vary across regions. A major advantage is the ability to accurately approximate $M$ with dramatically fewer pieces, without significantly increasing the number of parameters per piece. 

To give a flavor of the possible gains, we applied this approach to the economics data of Figure \ref{fig:economics}.  Based on cross validation, the optimal number of neighborhoods for local PCA and spherelets (local SPCA) were both $16$.  We color code the different pieces and show the fit in Figure \ref{fig:economics_1425} by two  pairwise plots to better visualize the learned curve. The fit is visually much better, limiting over-fitting and artifacts and capturing the smooth relationships more accurately.  This is also apparent from the out-of-sample MSE, which is $2.5\times 10^5$ for local PCA and $1.4\times 10^5$ for spherelets.

We provide strong theory support for spherelets. First, as a local result, for data generated from a sphere with measurement error, we show the rate of convergence of SPCA to the true sphere in Hausdorff distance as $n$ increases. As a global result, we provide the convergence rate for spherelets manifold estimates.  We additionally show a mathematical result bounding covering numbers for local spherical versus linear approximations, showing that dramatic improvements are possible for manifolds having sufficiently large curvature but not too many subregions having large changes in normal curvature. The sign of the Gaussian curvature has no impact on  approximation performance. In fact, spherelets work well for a broad class of Riemannian manifolds, regardless of whether the Gaussian curvature is positive or negative, or even varying in sign with location. We assume $d$ and $D$ are fixed throughout this paper, and all asymptotic theory is focused on increasing sample size $n$. 

Our simulation results show significant reductions in the number of components needed to provide an approximation within a given error tolerance.  SPCA has lower MSE than PCA and better preserves clusters. Comparing to local PCA, spherelet manifold estimates have much lower mean square error.  Substantial practical gains are shown in multiple examples. 

The paper is organized as follows. In section 2, we propose spherical principal component analysis (SPCA) and provide convergence rate theory. In section 3, we focus on manifold approximation, prove the convergence rate of spherelets and provide a  covering number theorem as additional support.  In section 4, we consider simulated and real data applications for both manifold estimation and data visualization. In section 5, we discuss some open problems and future work. Proofs justifying the SPCA estimator are in the Appendix, while other proofs are in the Supplementary Materials.

\section{Spherical  Principal Component Analysis}

\subsection{SPCA Algorithm}	\label{sec:SPCA}
Let $S_V(c,r) \coloneqq \{x: \|x-c \|=r, x-c\in V, \dim(V)=d+1\}$ be the $d$-dimensional sphere with center $c$ and radius $r$ lying in the $d+1$ dimensional affine subspace $c+V\subset \RR^D$. Throughout this paper, we do not distinguish between $V$ as a linear subspace and its matrix representation, that is, a matrix with orthonormal columns. Our goal in this section is to estimate 
$(V,c,r)$ to obtain the best approximating sphere through data $X_1,\ldots,X_n$ consisting of $n>d$ samples in $\RR^D$.
We first consider the projection of an arbitrary point $x$ to the sphere $S_V(c,r)$.
\begin{lemma}\label{sproj} For any $x\in \RR^D$, its projection to $S_V(c,r)$, or the closest point $y \in S_V(c,r)$, is  
	$$\underset{y\in S_{V}(c,r)}{\argmin}\  d^2(x,y)=\begin{cases}
		c+\frac{r}{\|VV^\top(x-c)\|}VV^\top(x-c) & VV^\top(x-c)\neq 0\\
		S_V(c,r) & VV^\top(x-c)= 0,
	\end{cases}$$
	where $c+VV^\top(x-c)$ is the projection of $x$ onto the affine subspace $c+V$.  
\end{lemma}
When $x=c$, the projection is not unique, but the $d(x,\mathrm{Proj}(x))$ is unique, which is the main focus of manifold approximation. As a result, we still treat the projection as unique without any confusion in the remaining sections. 

The projection of $x$ to the sphere in Lemma \ref{sproj} can be conducted by first projecting $x$ to a $d+1$ dimensional affine subspace and then further projecting to the sphere (see the proof of Lemma \ref{sproj}).  To find the optimal affine subspace, we use $d+1$-dimensional PCA, obtaining 
\begin{equation}
	\widehat{V} = (v_1,\cdots,v_{d+1}),\quad v_i=\mathrm{evec}_i\{(X-1\bar X^\top)^\top(X-1\bar X^\top)\}, \label{V_est}
\end{equation}
where $\mathrm{evec}_i(S)$ is the $i$th eigenvector of $S$ in decreasing order.  Letting $Y_i =\bar X+\widehat{V}\widehat{V}^{\top}(X_i-\bar X)$, we then find the optimal sphere through points $\{ Y_i \}_{i=1}^n$.  A sphere can be expressed as the set of zeros of a quadratic function $(y-c)^\top(y-c)-r^2$. When this quadratic function has positive value, $\|y-c\|>r$, so $y$ is outside the sphere, and $y$ is inside the sphere if the function has negative value. Hence, we define the loss function 
\begin{equation}\label{eqn:SPCAloss}
	\mathscr{L}(c,r)\coloneqq \sum_{i=1}^n\left((Y_i-c)^\top(Y_i-c)-r^2\right)^2.
\end{equation}
\begin{theorem}\label{SPCA}
	The minimizer of (\ref{eqn:SPCAloss}) is given by
	\begin{equation}\label{eqn:SPCAsolution}\widehat{c}=\frac{1}{2}H^{-1}\xi,~~\widehat{r}=\frac{1}{n}\sum_{i=1}^n\|Y_i-\widehat{c}\|
	\end{equation}
	where
	$H=\sum_{i=1}^n\big(Y_i-\bar Y)(Y_i-\bar Y\big)^\top$ and $
	\xi=\sum_{i=1}^n \Big(Y_i^\top  Y_i-\frac{1}{n}\sum_{j=1}^nY_j^\top Y_j\Big)\big(Y_i-\bar Y\big).$
\end{theorem}
We refer to the resulting estimates $(\widehat{V},\widehat{c},\widehat{r})$ as (empirical) {\em Spherical PCA} (SPCA). If we replace the sample mean by expectation, the  corresponding estimator is called population SPCA.

\begin{remark}\label{spcaremark}
	Alternatively, we could have minimized $\sum_{i=1}^n d^2(X_i,S_V(c,r)),$ corresponding to the sum of squared residuals, also known as geometric loss.  However, the resulting optimization problem is non convex, lacks an analytic solution, and iterative algorithms may be slow to converge, while only producing local minima. Instead, we consider the algebraic loss function in Equation (\ref{eqn:SPCAloss}), which is more robust with respect to noise \citep{coope1993circle} and admits a closed-form solution.
\end{remark}

\begin{corollary}\label{lossfunctions}
	If $X_i \in S_V(c,r)$ for all $i$, SPCA will find the same minimizer as the loss function in the above Remark, corresponding to exactly $(V,c,r)$.
\end{corollary}

The number of unknown parameters of $d$-dimensional PCA is $\mathrm{O}(Dd)$, while the number of unknown parameters of SPCA is $\mathrm{O}(Dd)+D+1=\mathrm{O}(Dd)$. In addition, the computational cost for the first step of SPCA is the same as PCA. The  additional cost of SPCA comes from calculating $c$ and $r$, which are both linear in $D$ and $n$ and hence dominated by PCA complexity. As a result, the two algorithms have the same order of computational cost. The key motivation for SPCA is to maintain simplicity, both conceptually and computationally, while improving performance by relaxing the linearity assumption.

\subsection{Asymptotics of SPCA}\label{sec:SPCAasym}

In this section, we show that if the data are concentrated around a sphere, then SPCA can recover this sphere with high probability. Assume $Y\sim \rho$ where $\supp(\rho)=S(V_0,c_0,r_0)$. Let $\epsilon\sim N(0,\sigma^2\Id_D)$ be Gaussian noise and $X=Y+\epsilon$. Let $X_1,\cdots,X_n$ be i.i.d observations and let the empirical solution of SPCA be $\widehat{V}_n,~\widehat{c}_n,~\widehat{r}_n$. 

We denote the population covariance matrix and sample covariance matrix by $\Sigma$ and $\hat{\Sigma}_n$, respectively. It is clear that all eigenvalues of $\Sigma$ are positive. Furthermore, we rely on the following assumption:

\begin{enumerate}[label=(A)]
	\item The first $d+1$ eigenvalues of $\Sigma$ are all distinct, denoted by $\lambda_1>\lambda_2>\cdots>\lambda_{d+1}> \sigma^2>0$ .
\end{enumerate}
\begin{theorem}\label{thm:SPCAerror}
	Under assumption (A), the following hold
	$$\|\widehat{V}_n-V_0\|\leq \mathrm{o}_p\left(\frac{\sigma\log n}{n^{1/2}}\right),~~\|\widehat{c}_n-c_0\|\leq \mathrm{O}(\sigma^2)+\mathrm{o}_p\left(\frac{\sigma\log n}{n^{1/2}}\right),~~|\widehat{r}_n-r_0|\leq \mathrm{O}(\sigma^2)+\mathrm{o}_p\left(\frac{\sigma\log n}{n^{1/2}}\right).$$
	
\end{theorem}

The above theorem provides an upper bound on the error rate in estimating each sphere parameter. Given observations with measurement error, SPCA can recover the true parameters of the sphere at the parametric rate with respect to the sample size up to a log factor and an asymptotic error depending on the measurement error variance. The following corollary controls the Hausdorff distance (denoted by $d_H$) between the true sphere and the estimated sphere by SPCA.
\begin{corollary}\label{cly:SPCAHausdorff}
	Under the same assumption as Theorem \ref{thm:SPCAerror},
	$$d_H(S(V_0,c_0,r_0),S(\widehat{V}_n,\widehat{c}_n,\widehat{r}_n))\leq C\sigma^2+\mathrm{o}_p\left(\frac{\sigma \log n}{n^{1/2}}\right).$$
\end{corollary}
For previous theoretical results on manifold estimation under Hausdorff loss, refer to 
\cite{genovese2012manifold,kim2015tight}.  It is typical in the literature on asymptotic theory for manifold approximation to assume that either the data are noiseless \citep{fefferman2019fitting,aamari2019nonasymptotic,sober2019manifold,genovese2012manifold}, the noise is perpendicular to the manifold and bounded \citep{genovese2012minimax}, or the level of measurement error decreases with the sample size \citep{fefferman2018fitting, GMRA2016,aamari2019nonasymptotic}; this would allow us to remove the asymptotic bias in the above bounds.  

\section{Manifold Approximation}
Most manifolds cannot be adequately approximated by a single PCA or SPCA.  Hence, in this section, we consider using local SPCA to approximate the manifold locally by spherelets.

\subsection{Local SPCA}\label{sec:localSPCA}
Assume $Y\sim \rho$ where $\supp(\rho)=M$. Let $\epsilon\sim N(0,\sigma^2\Id_D)$ be Gaussian noise and $X=Y+\epsilon$ with i.i.d observations $X_1,\cdots,X_n$. A single sphere will typically not be sufficient to approximate the entire manifold $M$, but instead we partition $\RR^D$ into local neighborhoods and implement SPCA separately in each neighborhood.  This follows similar practice to popular implementations of local PCA, but we apply SPCA locally instead of PCA.  We divide $\RR^D$ into non-overlapping subsets $C_1,\ldots,C_K$.  For the $k$th subset, we let $X_{[k]} = \{ X_i: X_i \in C_k \}$, $(\widehat{V}_k,\widehat{c}_k,\widehat{r}_k)$ denote the results of applying SPCA to data $X_{[k]}$, $\mathrm{Proj}_k$ denote the projection map from $x \in C_k$ to $y \in S_{\widehat{V}_k}(\widehat{c}_k,\widehat{r}_k)$ obtained by Lemma \ref{sproj}, and $\widehat{M}_k = S_{\widehat{V}_k}(\widehat{c}_k,\widehat{r}_k)\cap C_k $.  Then, we approximate $M$ by $\widehat{M} = \bigcup_{k=1}^K \widehat{M}_k$.  

In general, $\widehat{M}$ will not be continuous or a manifold but instead is made up of a collection of pieces of spheres chosen to approximate the manifold $M$.  There are many ways in which one can choose the subsets $\{ C_k \}_{k=1}^K$, but in general the number of subsets $K$ will be chosen to be increasing with the sample size with a constraint so that the number of data points in each subset cannot be too small, as then $\widehat{M}_k$ cannot be reliably estimated.  Below we provide theory on mean square error properties of the estimator $\widehat{M}$ under some conditions on how the subsets are chosen but without focusing on a particular algorithm for choosing the subsets.

There are a variety of algorithms for multiscale partitioning of the sample space, ranging from cover trees \citep{cover2006} to METIS \citep{METIS1998}, to iterated PCA \citep{iteratedPCA2009}. As the scale becomes finer, the number of partition sets increases exponentially and the size of each set decreases exponentially. Assume $U\subset M$ is an arbitrary submanifold of $M$, $\rho_U=\rho|_U$ is the probability measure of data $X_i$ conditionally on $X_i\in U$ and $\mathrm{diam}(U)=\sup_{x,y\in U} d(x,y)=\alpha_U.$
For example, if we bisect the unit cube in $\RR^D$ $j$ times, then the diameter of each piece will be $\alpha_U\propto 2^{-j}$ {which decays to zero with $j$}. The approximation error depends on $\alpha_U$: as $\alpha_U\rightarrow0$, each local neighborhood is smaller so linear or spherical  approximations perform better.

\begin{theorem}\label{thm:mfderror}
	Addition to assumption (A), assume \newline
	(B): There exists $\delta>0$ such that for any submanifold $U\subset M$, $r_U\geq \delta$, where $r_U$ is the radius of the sphere obtained by population SPCA {(defined after Theorem \ref{SPCA})} on $U$ with respect to the measure $\rho_U$. \newline 
	(C): The partition $\{C_1,\cdots,C_K\}$ is regular in the sense that $\mathrm{diam}(C_k)^d=\mathrm{O}(n_k/n)$, where $n_k$ is the number of samples in $C_k$. \newline 
	Then the manifold approximation uniform error rate is
	$$\sup_{x\in M}d(x,\widehat{M})\leq  C\sigma^2+o_p\left(\frac{\sigma\log n}{n^{\frac{2}{d+4}}}\right).$$
\end{theorem}

As discussed after Corollary \ref{cly:SPCAHausdorff}, under stronger assumptions on the noise, the bias term $\sigma^2$ can be removed so that the rate is $n^{-\frac{2}{d+4}}$ up to a $\log$ factor while the optimal rate for manifold approximation is $n^{-\frac{2}{d+2}}$ \citep{genovese2012minimax}.
Assumption (B) is a very weak regularity condition on the manifold $M$, which rules out extreme kinks in $M$ leading to unbounded curvature and hence arbitrarily small radius $r_U$ of the best fitting sphere in a local region $U \subset M$ containing the kink. Assumption (C) is also {reasonable}. For example, suppose  
the density function of $\rho$, denoted by $f_\rho$, is strictly positive. By compactness of $M$, $f_\rho$ is bounded above and below by a positive number, and hence $\rho(C_k)\sim \mathrm{Vol}(C_k)\sim\mathrm{diam}(C_k)^d$. Since $\rho(C_k)\sim n_k/n$, (C) follows. { In addition, throughout this paper, we assume $D$ and $d$ are fixed and all asymptotic theories are for $n\to\infty$. 
}

\subsection{Covering Numbers}\label{sec:covering}
Theorem \ref{thm:mfderror} does not imply that applying SPCA in local neighborhoods leads to a better rate than applying PCA. This is not surprising since we are not restricting the curvature. When the curvature is zero, we expect SPCA and PCA to have similar performance. However, when curvature is not approximately zero, SPCA is expected to have notably improved performance except for very small local regions. This is consistent with the empirical results in the following section. 	In this section, we provide geometric evidence in favor of spherelets over local PCA through covering numbers. This covering number theory is mathematical and does not involve data or distributional assumptions.

We define the covering number as the minimum number of local bases needed to approximate the manifold within $\epsilon$ error.  Our main theorem shows the covering number of spherelets is smaller than that of hyperplanes. We assume $M$ to be compact, otherwise the covering number is not well-defined. 
\begin{definition}
	Let $M$ denote a $d$-dimensional compact $C^3$ Riemannian manifold embedded in $\RR^D$, and $\mathcal{B}$ a 
	collection of $d$-dimensional subsets of $\RR^D$.  Then the $\epsilon>0$  covering number $N_\mathcal{B}(\epsilon,M)$ is defined as
	\begin{eqnarray*}
		N_\mathcal{B}(\epsilon,M)\coloneqq \inf_{K\in \ZZ_+}\bigg\{K: \exists\{C_k, \mathrm{Proj}_k,B_k\}_{k=1}^K\text\ s.t.\ \Big\|x-\mathrm{Proj}(x)\Big\|\leq \epsilon,\ \forall x\in M\bigg\},
	\end{eqnarray*}
	where $\{C_k\}_{k=1}^K$ is a partition of $\RR^D$, $B_k\in\mathcal{B}$, $\mathrm{Proj}_k: C_k\rightarrow B_k,\ x\mapsto \underset{y\in B_k}{\argmin} \|x-y\|^2$ is the corresponding local projection and $\mathrm{Proj}(x)\coloneqq\sum_{k=1}^K \mathbf{1}_{\{x\in C_k\}} \mathrm{Proj}_k (x)$ is the global projection. 
\end{definition}

The above covering number is the minimal number of bases in dictionary $\mathcal{B}$ needed to approximate 
$M$ with $\epsilon$ error. We focus on two choices of $\mathcal{B}$: all $d$-dimensional hyperplanes in $\RR^D$, denoted by $\mathcal{H}$, and all $d-$dimensional spheres in $\RR^D$, denoted by $\mathcal{S}$. 
Including spheres with infinite radius in $\mathcal{S}$, we have  $\mathcal{H}\subset\mathcal{S}$ implying the following Proposition.
\begin{proposition} For any compact $C^3$ Riemannian manifold $M$, and $\epsilon>0$, 
	$N_{\mathcal{S}}(\epsilon,M)\leq N_{\mathcal{H}}(\epsilon,M).$
\end{proposition}
The proposition implies that an oracle focused on approximating $M$ with error $\le \epsilon$ using either spheres or hyperplanes will
never need to use more spheres than hyperplanes.  The following Theorem provides an improved comparison.

\begin{theorem}[Covering Number]\label{mainthm}
	Assume $M$ is a compact $C^3$ $d$-dimensional Riemannian manifold. Then there exists constants $C=C(M)$ and $\delta=\delta(M)>0$ such that $\forall \epsilon\leq \delta$
	\begin{equation}\label{eqn:coveringnumber}
		N_{\mathcal{H}}(\epsilon,M)\leq C\epsilon^{-\frac{d}{2}},\qquad
		N_\mathcal{S}(\epsilon,M)\leq CV_\epsilon \epsilon^{-\frac{d}{3}}+C(V-V_\epsilon)\epsilon^{-\frac{d}{2}},
	\end{equation}
	where $V$ is the (Riemannian) volume of $M$ while $V_\epsilon\in[0,V]$ is the volume of a submanifold of $M$ that is locally well approximated by spheres; see Appendix \ref{sec:Thm8} for more details.
\end{theorem}
When $d=1$, $M=\gamma$ is a curve and we have the following Corollary.
\begin{corollary} \label{clr:coveringcurv}For any $\epsilon>0$ and compact $C^3$ curve $\gamma$, 
	$N_\mathcal{S}(\epsilon,\gamma)\leq C\epsilon^{-\frac{1}{3}}$.
\end{corollary}
\begin{remark}
	The  curse of dimensionality comes in through the term $\epsilon^{-\frac{d}{2}}$, but we can decrease its impact from $\epsilon^{-\frac{d}{2}}$ to $\epsilon^{-\frac{d}{3}}$ using spheres instead of planes.
\end{remark}
\begin{proposition}\label{tight}
	The upper bounds of covering numbers  $N_{\mathcal{H}}(\epsilon,M)$ and $N_{\mathcal{S}}(\epsilon,M)$ are both tight {in terms of the rate in $\varepsilon$}.
\end{proposition}

The covering number depends on the geometry of the manifold, particularly the curvature of geodesics on the manifold and not the sectional curvature or Ricci curvature.  Although spheres have positive Gaussian/sectional curvature, they can be used as a dictionary to estimate manifolds with negative Gaussian/sectional curvature. For example when the manifold is a 2-dimensional surface, $V_\epsilon$ is determined by the difference of the two principal curvatures, not the Gaussian curvature or mean curvature.  When $d=1$, the absolute curvature determines the radius of the osculating circle and the sign of the curvature determines which side of the tangent line the center of the circle is on. This provides intuition for why spherelets works well for both positive and negative curvature spaces. 

The bounds in our main theorem are {\em tight}, implying that spherelets often require many fewer pieces than locally linear dictionaries to approximate $M$ to any fixed accuracy level $\epsilon$; particularly large gains occur when a non-negligible subset of $M$ is covered by the closure of points having not too large change in curvature of geodesics along different directions.  As each piece involves $O(D)$ unknown parameters, these gains in covering numbers should lead to real practical gains in statistical performance; indeed this is what we have observed in applications.

\section{Applications, Algorithms and Examples}\label{sec:experiments}
This section contains a variety of simulation studies and real data applications of SPCA and spherelets.
To measure performance in analyzing data, we focus on the mean squared error (MSE),  $\frac{1}{n}\sum_{i=1}^n \|X_i-\widehat{X}_i\|$, where $\widehat{X}_i$ is the fitted value of $X_i$ by a dimensionality reduction method. Most dimension reduction methods do not provide $\widehat{X}_i$ but instead replace $X_i$ with a lower-dimensional summary; such methods are not directly comparable to global or local SPCA.

\subsection{Simulation study of global SPCA}

We first verify the convergence rates in Theorem \ref{thm:SPCAerror} on simulated data. For a specific ambient and intrinsic dimension pair $(D,d)$, 
we sample $Y_1,\cdots,Y_n$ from a von Mises-Fisher distribution on $S^d$ with sample size $n$. Then we sample Gaussian noise $\epsilon_i\sim N(0,\sigma^2 \Id_D)$, and let $X_i = c_0+r_0V_0Y_i+\epsilon_i$, where $c_0\in\RR^{D}$, $r_0>0$ and $V_0\in \RR^{D\times(d+1)}$ with orthonormal columns. The $X_i$ are close to the sphere centered at $c_0$ with radius $r_0$ in linear subspace $V_0$. The true parameters are generated randomly. 

First we fix $\sigma^2 = \{0.1,0.01\}$ {(the results are similar for larger $\sigma^2$ so we present the plots for two different $\sigma^2$ to make the figure clearer)}. Let $\widehat{V}$, $\widehat{c}$ and $\widehat{r}$ be the estimated parameters by SPCA. Figure \ref{fig:SPCA_rates_n} shows the error rate with respect to sample size $n$ for two choices of $(D,d)$. The  $x$-axis is $\log$ sample size while the $y$-axis is $\log(\|V_0-\widehat{V}\|)$, $\log(\|c_0-\widehat{c}\|)$ and $\log(|r_0-\widehat{r}|)$. The rates for all three parameters are $n^{-1/2}$. In Theorem \ref{thm:SPCAerror}, when $\sigma^2$ is fixed, the upper bound of the error rate for $V$ is $n^{-1/2}$ up to a $\log n$ factor, which is asymptotically negligible compare to $\sqrt{n}$. While the rates for $c$ and $r$ are affected by the bias term $\mathrm{O}(\sigma^2)$, the rates are very close to $n^{-1/2}$ when $C_{v}\log n/n^{1/2}\gg C_b\sigma$, 
where $C_b$ and $C_v$ are the constants in Theorem \ref{thm:SPCAerror} for the bias and variance, respectively. 

There is a subtle difference between the rate of $c$ and $r$ because of the constant $C_b$ that controls the bias. For the center $c$, $C_b\sim 1/\lambda_{d+1}^2$ where $\lambda_{d+1}$ is the $(d+1)$-th eigenvalue of the covariance matrix of $Y$, which is supported on the $d$-dimensional sphere without noise. As a result, $\lambda_{d+1}$ 
is often much larger than $\sigma^2$; otherwise $Y$ will be concentrated around the equator of $S^d$ or $S^{d-1}$.  $\lambda_{d+1}/\sigma$ can be viewed as the ``signal-to-noise ratio'' (SNR). The constant $C_b$ in the bias term for $c$ is proportional to $1/\mathrm{SNR}^2$, which is usually very small; hence the rate for $c$ is observed as $n^{-1/2}$ in our experiments. For radius $r$, the constant $C_b\sim \sqrt{d+1}$, which is independent of the SNR and often larger than the constant for $c$. As a result, we may observe attenuation of the $n^{-1/2}$ rate for $r$ when $n$ is large enough so that $C_{v}\log n/n^{1/2}\lesssim C_b\sigma$. For more details of the constant $C_b$, see the proof of Theorem \ref{thm:SPCAerror} in the Supplementary Materials.
\begin{figure}[!h]
	\centering
	\includegraphics[width=0.49\textwidth,height=150pt]{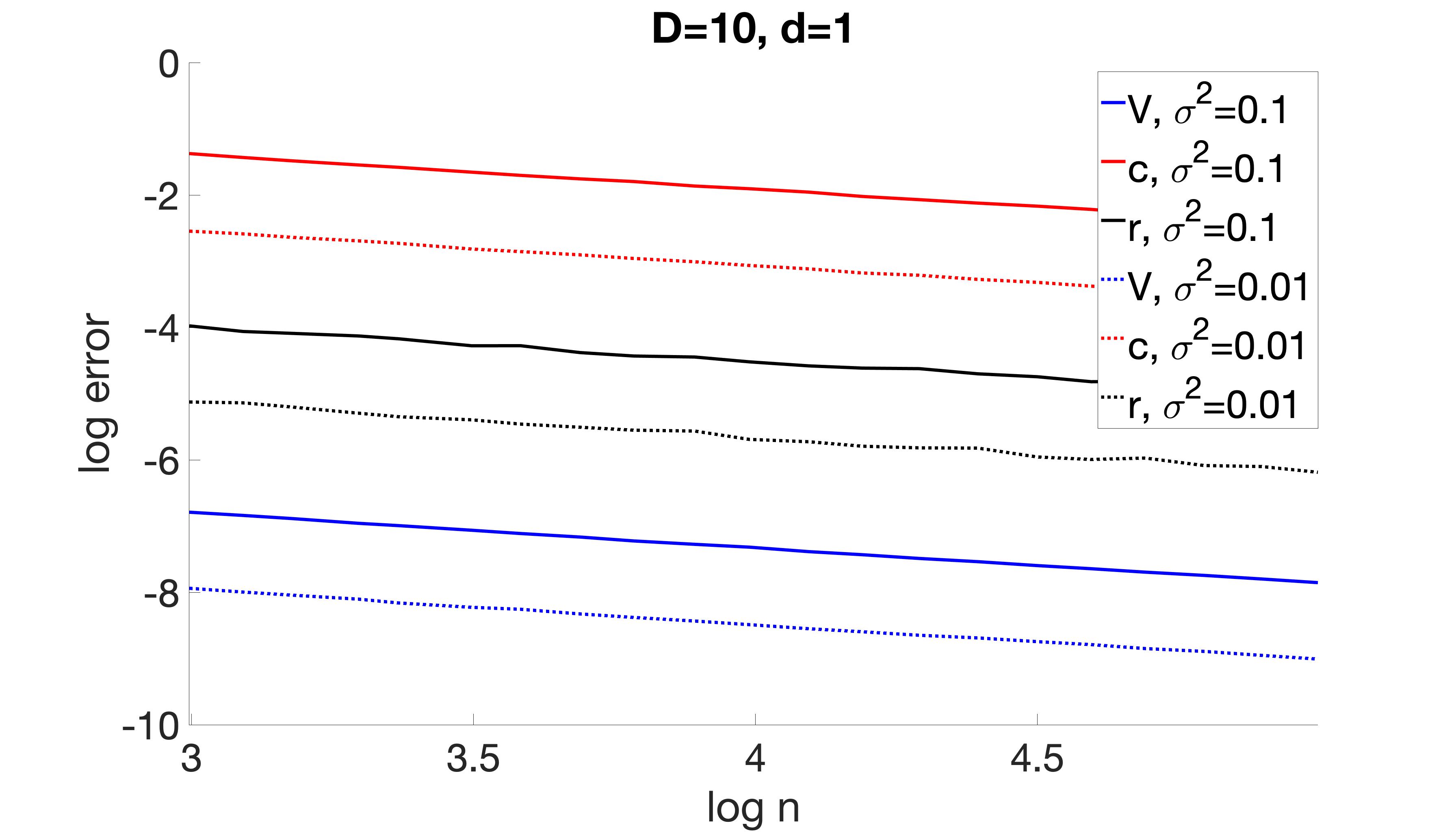}
	\includegraphics[width=0.49\textwidth,height=150pt]{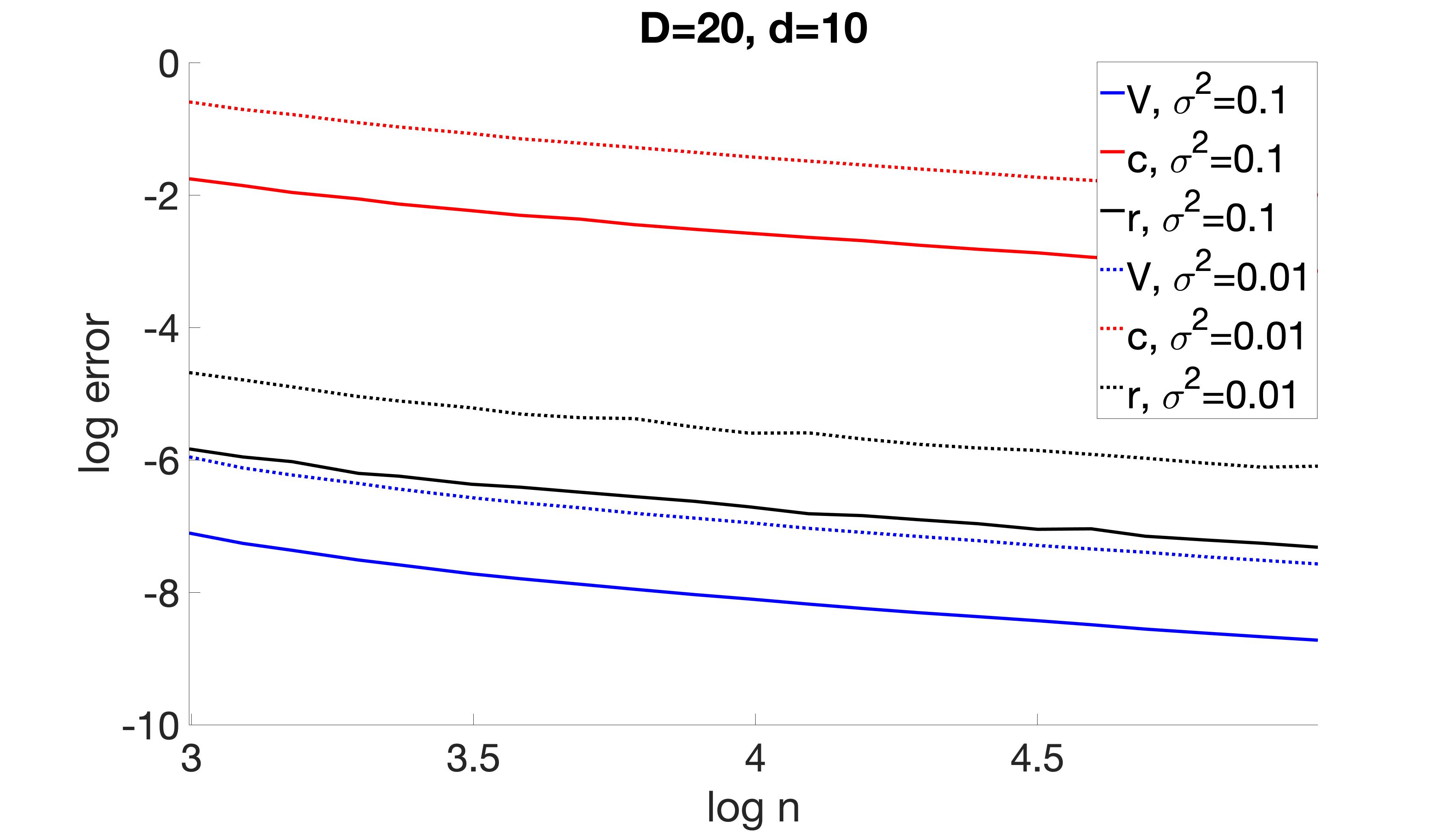}
	\caption{Error rates in estimating the sphere parameters $V$ (blue), $c$ (red), and $r$ (black) for $\sigma^2=0.1$ (solid lines) and $\sigma^2=0.01$ (dotted lines) as $n$ increases. The left panel is for a circle, and the right a 10-dimensional hyper-sphere.}
	\label{fig:SPCA_rates_n}
\end{figure}

Next we fix $n\in\{100,1000\}$ and vary $\sigma^2$. Figure \ref{fig:SPCA_rates_sigma2} shows that the error rates in estimating the sphere parameters as $\sigma^2$ varies are all $(\sigma^2)^{1/2}$. In Theorem \ref{thm:SPCAerror}, when $n$ is fixed, the upper bound of the error rate with respect to $\sigma$ is $\mathrm{O}\left(\sigma^2+\sigma\right)=\mathrm{O}\left(\sigma\right)$ since $\sigma$ dominates $\sigma^2$ when $\sigma\to 0$. 

\begin{figure}[!h]
	\centering
	\includegraphics[width=0.49\textwidth,height=150pt]{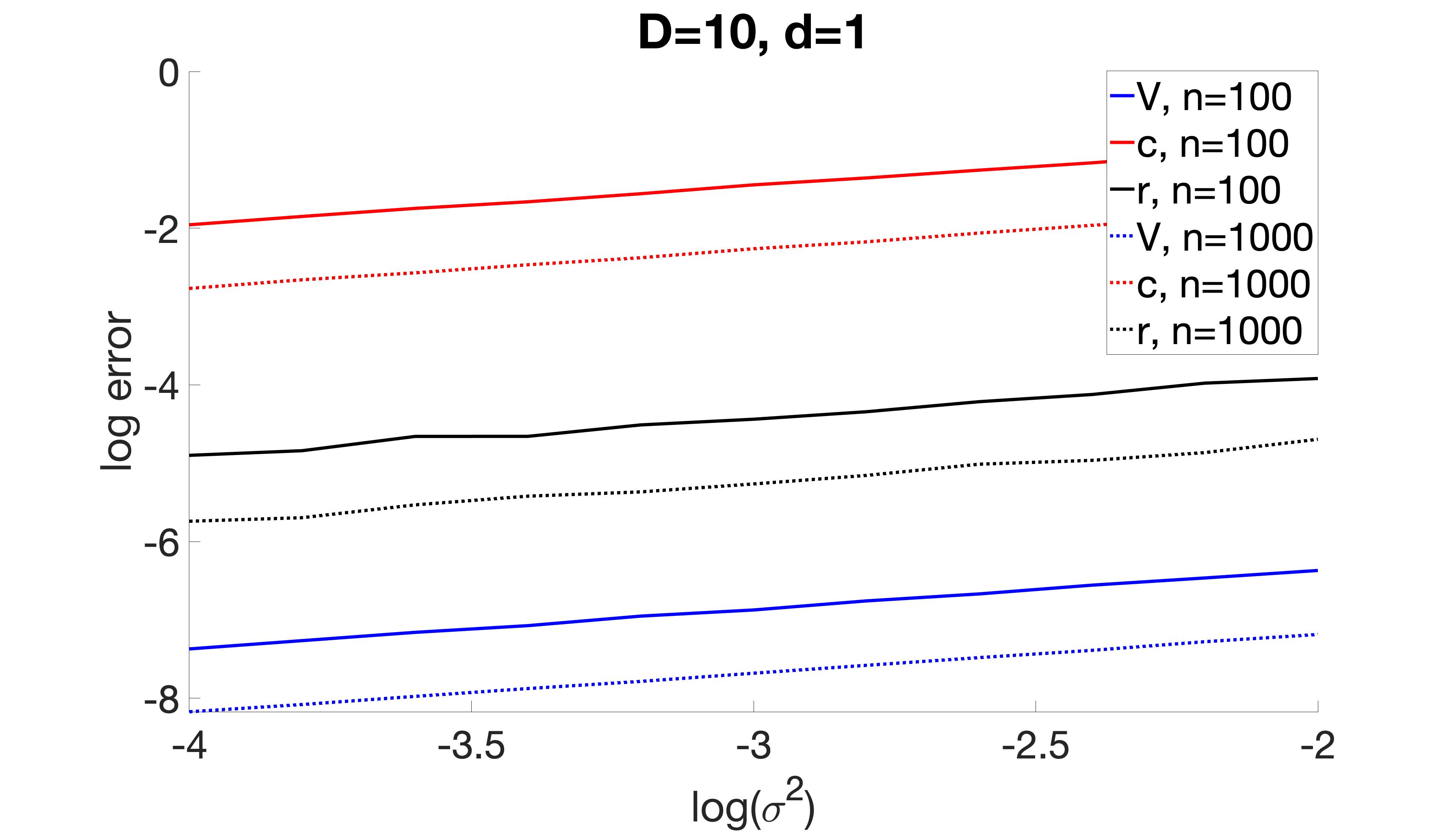}
	\includegraphics[width=0.49\textwidth,height=150pt]{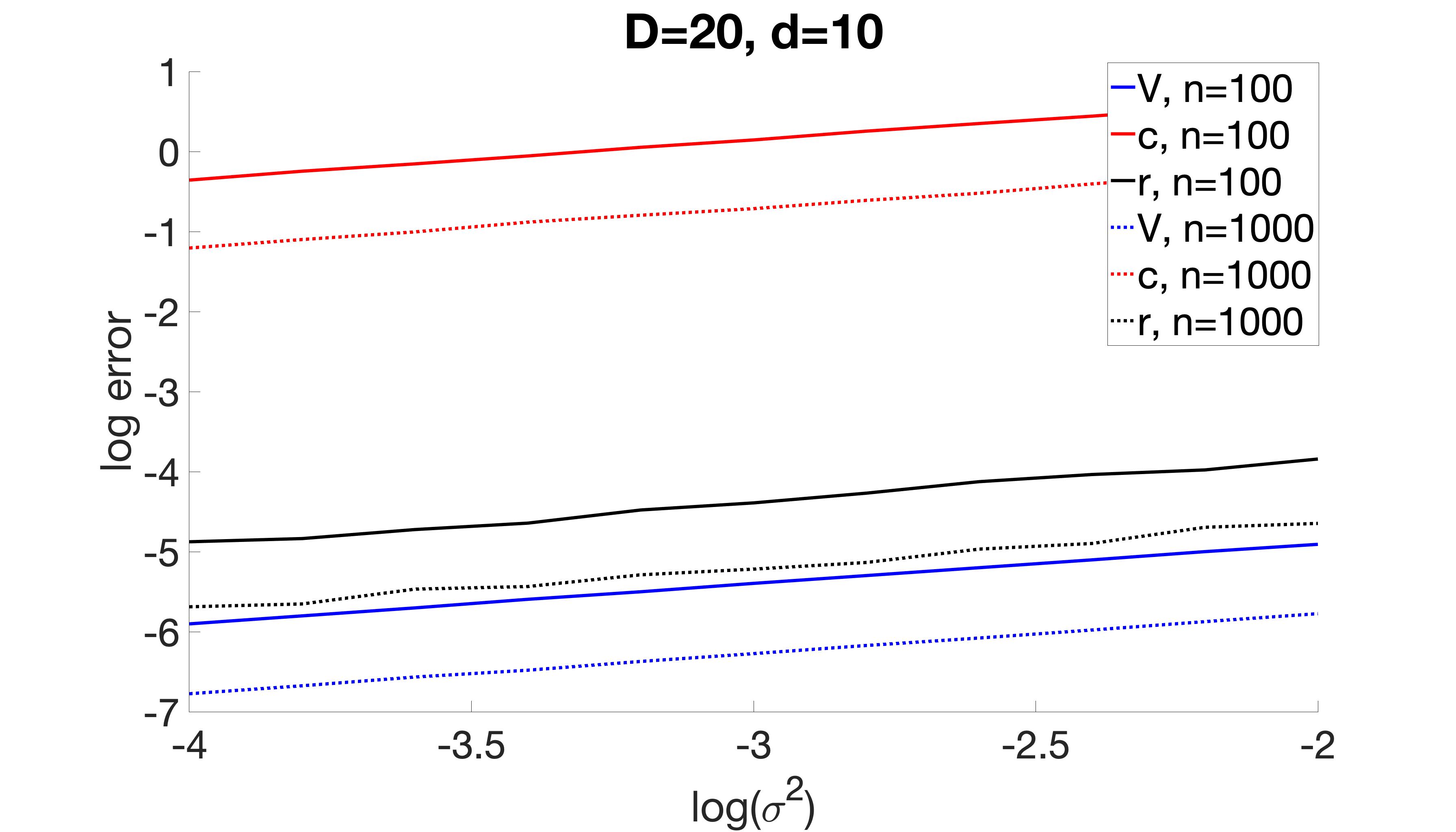}
	\caption{Error rates in estimating the sphere parameters $V$ (blue), $c$ (red), and $r$ (black) for $n=100$ (solid lines) and $n=1000$ (dotted lines) as $\sigma^2$ increases. The left panel is for a circle, and the right a 10-dimensional hyper-sphere.}
	\label{fig:SPCA_rates_sigma2}
\end{figure}

\subsection{Application to data visualization}

A common focus of dimensionality reduction algorithms is on data visualization.  To illustrate the use of SPCA for data visualization we consider an application to a banknote dataset.  The data consist of 1372 $400 \times 400$ pixel images of genuine and fake banknotes.  Based on these images, 4 features are extracted using a wavelet analysis and it is of interest to investigate differences between authentic and fake banknotes.  For data visualization, we choose $d=2$ for all algorithms, and compare SPCA with PCA, t-distributed stochastic neighbor embedding (tSNE, \cite{tSNE2008}), uniform manifold approximation and projection (UMAP, \cite{mcinnes2018umap}), locally linear embedding (LLE, \cite{lle2000}) and multidimensional scaling (MDS, \cite{kruskal1964multidimensional}). 

In Figure \ref{fig:Cluster_Banknote}
we plot the 2-dimensional embedding of tSNE, UMAP, LLE, MDS, PCA, and SPCA. For PCA we show the first two principal components.  For SPCA, we first project the data to the 2-dimensional sphere and then obtain the polar angle and azimuthal angle.

The results are shown in Figure \ref{fig:Cluster_Banknote}, which shows a clear separation between the genuine and fake banknotes for SPCA.  Interesting, (locally) linear methods including PCA, LLE and MDS fail to separate the two groups, at least based on only two components.  The popular tSNE and UMAP methods do well at separating genuine and forged banknotes, but in a highly complex way that shows many sub-clusters in the data.  These approaches are much more complex that SPCA, including computationally, and there has been concern in the literature that tSNE and UMAP may show artifactual structure in the data \citep{wattenberg2016use}.  

\begin{figure}[!h]
	\centering
	\includegraphics[width=\textwidth,height=150pt]{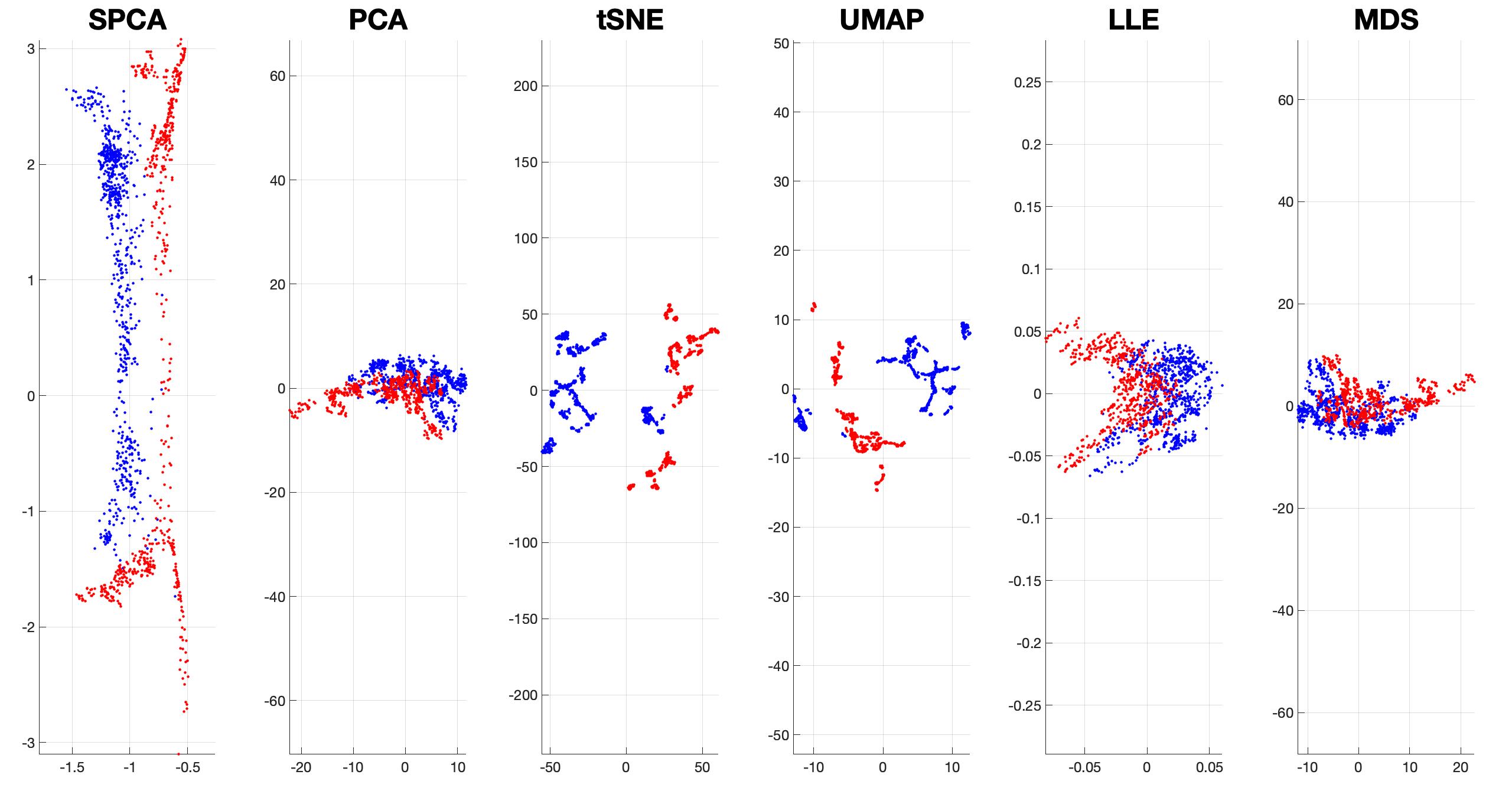}
	\caption{2-dimensional representation of Banknote data by six dimension reduction algorithms with two clusters: blue for authentic and red for forged}
	\label{fig:Cluster_Banknote}
\end{figure}


\subsection{Manifold approximation via spherelets} \label{manifoldapprox}
In manifold approximation we attempt to find an estimator of the unknown manifold $M$, say $\widehat M$. When the underlying manifold is complex, a single sphere is not enough, motivating local SPCA or spherelets.
As local SPCA provides an estimator of a submanifold $U\subset M$ in a neighborhood, we split $\RR^D$ into subsets 
$C_1,\ldots,C_k$
and apply local SPCA to estimate the manifold in each subset. 

Let $M_k=C_k\cap M$ be the sub-manifold of $M$ restricted to $C_k$. Let 
$\widehat{M}_k$ denote the estimate of $M_k$ based on applying SPCA to the data within $C_k$, and 
set $\widehat{M}=\bigcup_{k=1}^K \widehat{M}_k$. The map which projects a data point $x$ to the estimated manifold $\widehat{M}$ is denoted by $\widehat{\mathrm{Proj}}: \mathbb{R}^D\rightarrow \widehat{M}$. Algorithm \ref{MSEalg} describes the calculation of $\widehat{\mathrm{Proj}}$ and $\widehat{M}$ given a partition of $\RR^D$.

\begin{algorithm}
	\SetKwData{Left}{left}\SetKwData{This}{this}\SetKwData{Up}{up}
	\SetKwFunction{Union}{Union}\SetKwFunction{FindCompress}{FindCompress}
	\SetKwInOut{Input}{input}\SetKwInOut{Output}{output}
	\Input{Data $X_1,\cdots,X_n$; intrinsic dimension $d$; Partition $\{C_k\}_{k=1}^K$}
	\Output{
		The Estimated manifold $\widehat{M}$ of $M$ and the projection map $\widehat{\mathrm{Proj.}}$
	}
	\BlankLine
	\For{$k=1:K$}{
		\emph{Define $X_{[k]}=X \cap C_k$}\;
		\vskip2pt
		\emph{Calculate $\widehat{V}_k,\widehat{c}_k,\widehat{r}_k,$ by \eqref{eqn:SPCAsolution}}\;
		\vskip2pt
		\emph{Calculate $\widehat{\mathrm{Proj}}_k(x)=\widehat{c}_k+\frac{\widehat{r}_k}{\|\widehat{V}_k \widehat{V}_k^\top(x-\widehat{c}_k)\|}(x-\widehat{c}_k)$}\;
		\vskip2pt
		\emph{Calculate $\widehat{M}_k=S_{\widehat{V}_k}(\widehat{c}_k,\widehat{r}_k)\cap C_k $}\;
	}
	\emph{Calculate $\widehat{\mathrm{Proj}}(x)=\sum_{k=1}^K {\bf1}_{\{x\in C_k\}}\widehat{\mathrm{Proj}}_k(x)$, and $\widehat{M}=\bigcup_{k=1}^K \widehat{M}_k$}.
	\caption{Spherelets algorithm to estimate the manifold and projection map by applying local SPCA.}
	\label{MSEalg}
\end{algorithm}

We apply spherelets to multiple examples.  The first two (Euler spiral, cylinder) are toy examples using 
knowledge of the manifold to choose the partition. The subsequent examples (Euler spiral, economics, user knowledge) use a multiscale scheme to choose $C_1,\ldots,C_k$, and compare with local PCA.

{\bf Euler Spiral}. The Euler spiral,
$\gamma(s)=\left[ \int_{0}^s \cos(t^2)\mathrm{dt}, \int_{0}^s \sin(t^2)\mathrm{dt}\right], s\in [0,2],$
is a common example in the manifold learning literature, having curvature linear with respect to the arc length $s$. We generate $s_i$ uniformly from $[0,2]$ and then add Gaussian noise to $\gamma(s_i)$.  We uniformly partition $s\in[0,2]$ to obtain 
$C_1,\ldots,C_k$. The first panel in Figure \ref{fig:rate_Spiral} shows the convergence rate is $n^{-0.57}$, which is better than the $n^{-0.4}$ rate in Theorem \ref{thm:mfderror}. 
The remaining panels show the projected data with different sample size and different number of partitions{, along with the true spiral}. When the number of partitions is $3$, the approximation performance is excellent.

\begin{figure}[!h]
	\centering
	\includegraphics[height = 120pt,width=0.33\textwidth]{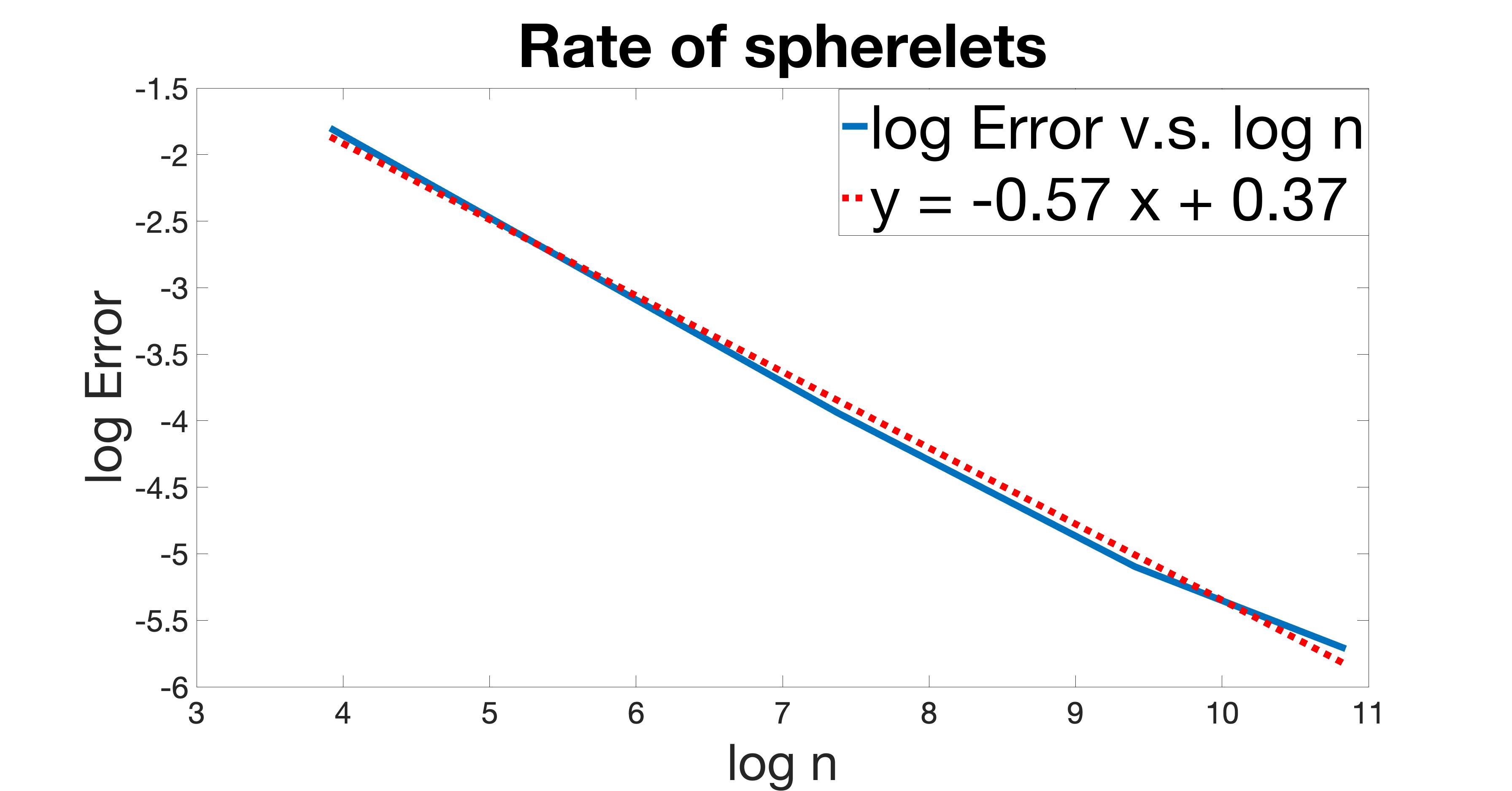}
	\includegraphics[height = 120pt,width=0.3\textwidth]{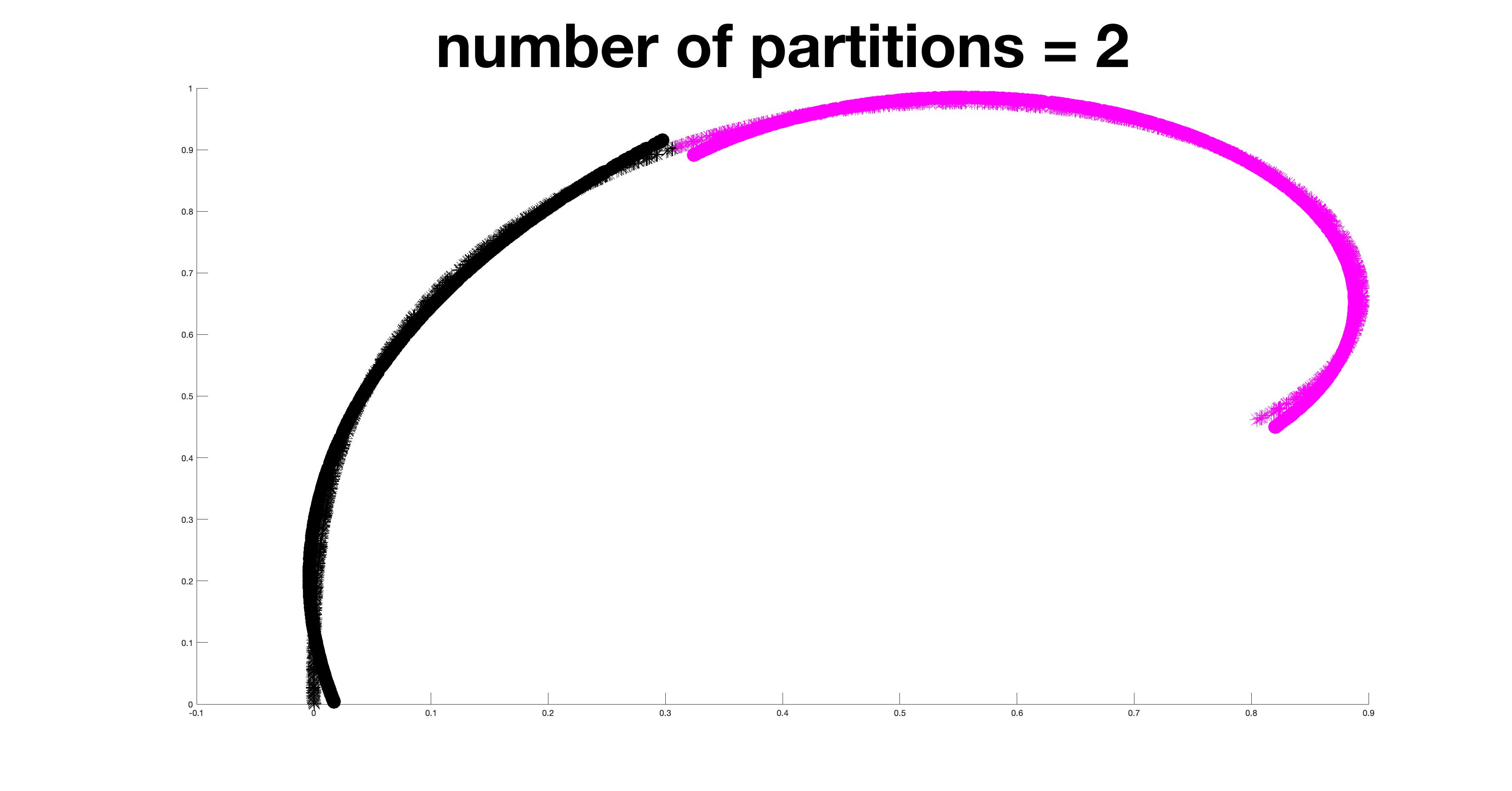}
	\includegraphics[height = 120pt,width=0.33\textwidth]{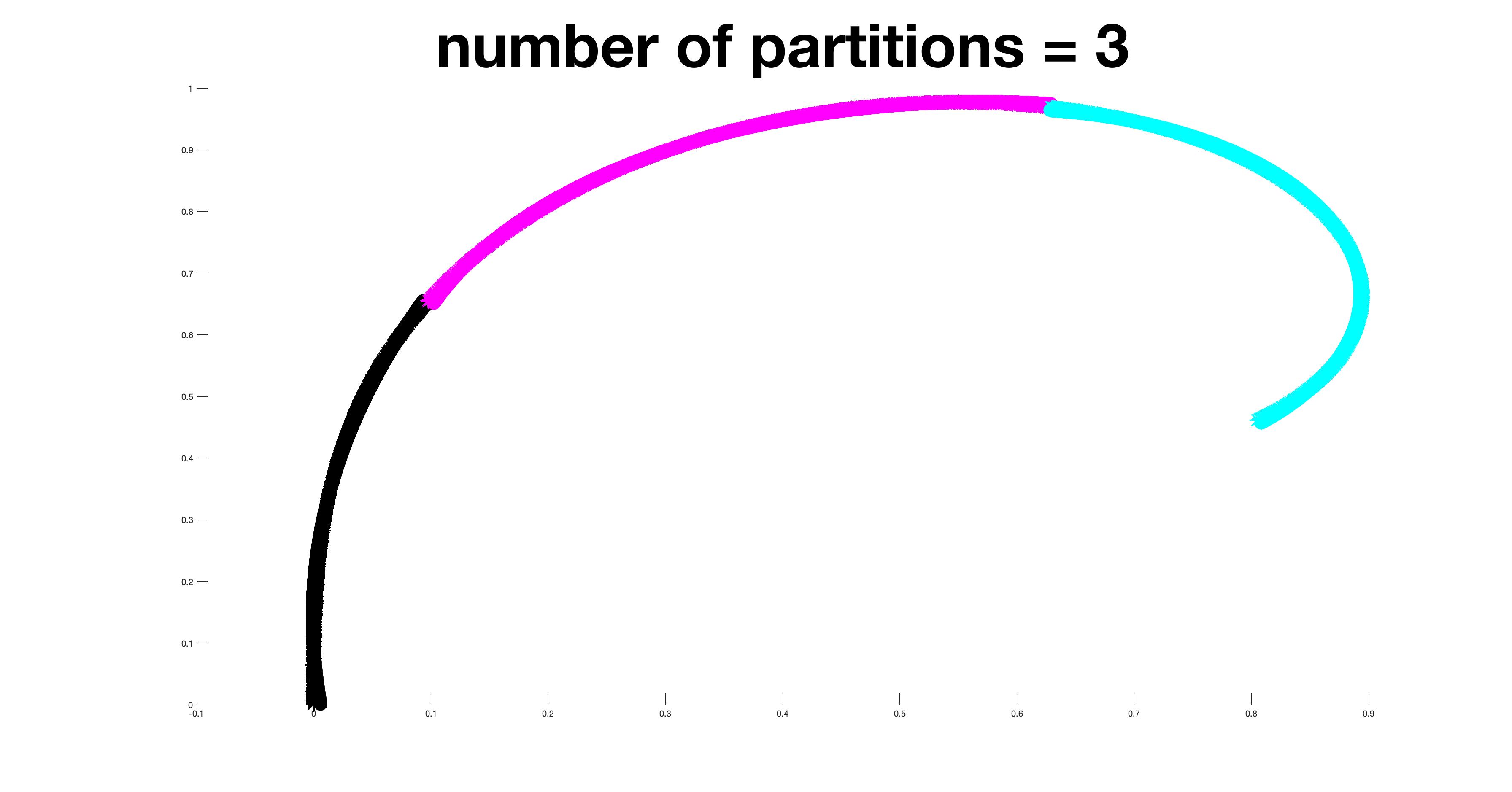}
	\caption{Spherelets on Euler spiral, left: convergence rate of MSE w.r.t. sample size; mid: projected samples to the estimated manifold with 2 circles; right: projected samples to the estimated manifold with 3 circles; colors
		correspond to different local neighborhoods.}
	\label{fig:rate_Spiral}
\end{figure}

{\bf Cylinder}. The cylinder
$\{(x=\cos\theta,y=\sin\theta,z):\theta\in[0,2\pi],z\in[0,1]\}$ is a surface with zero Gaussian curvature. The principal curvatures are $0$ and $1$. We independently sample $\theta_i$ and $z_i$ uniformly from $[0,\pi]$ and $[0,1]$, respectively. $C_1,\ldots,C_k$ are obtained by uniformly partitioning $z\in[0,1]$. The first panel in Figure \ref{fig:rate_cylinder} shows the convergence rate is $n^{-0.35}$, slightly better than the $n^{-1/3}$ rate in Theorem \ref{thm:mfderror}. The last two panels show how spherelets approximate the cylinder, where the number of partitions is $2$ and $3$, respectively.

\begin{figure}[!h]
	\centering
	\includegraphics[height = 120pt,width=0.33\textwidth]{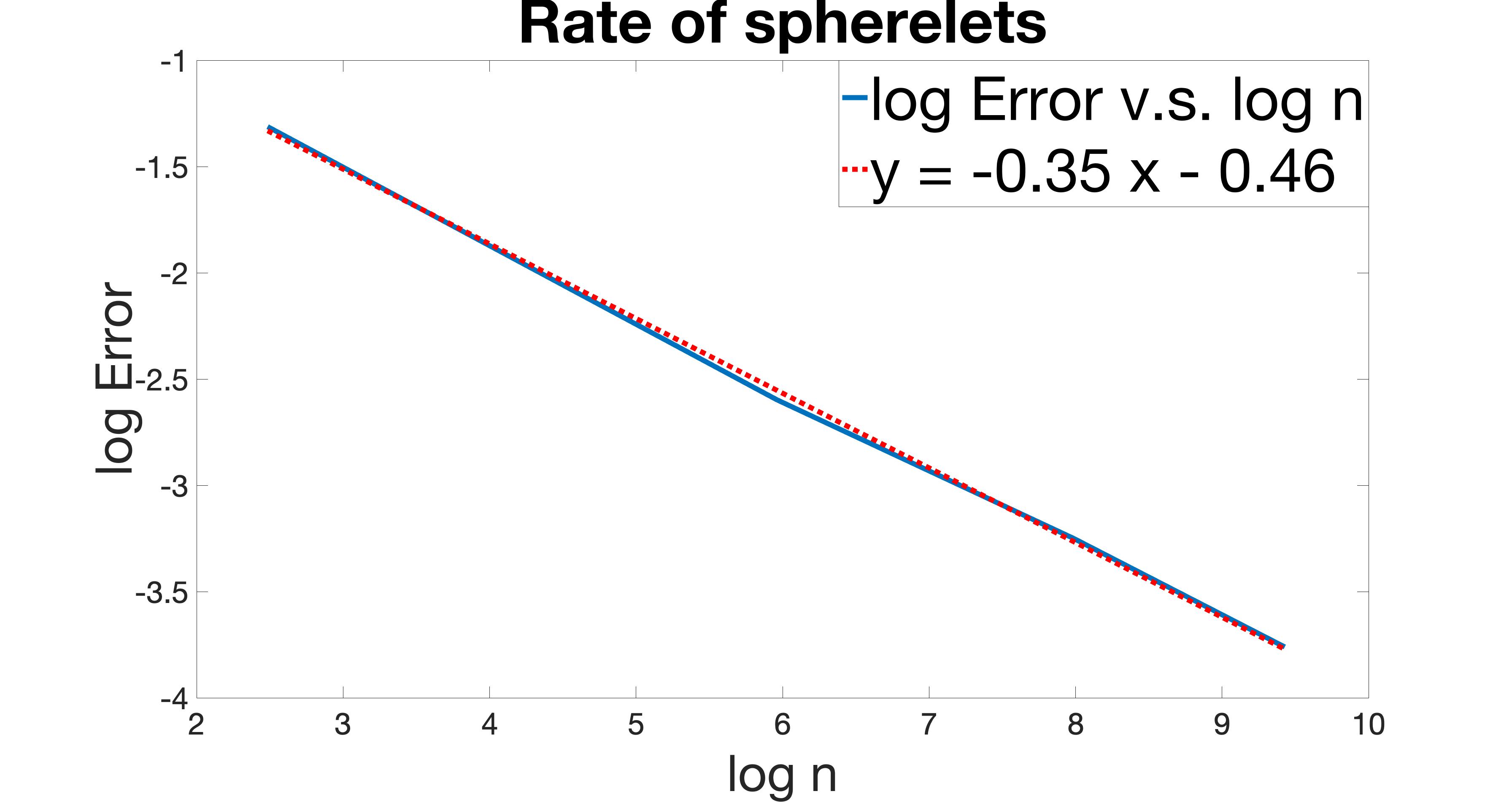}
	\includegraphics[height = 120pt,width=0.3\textwidth]{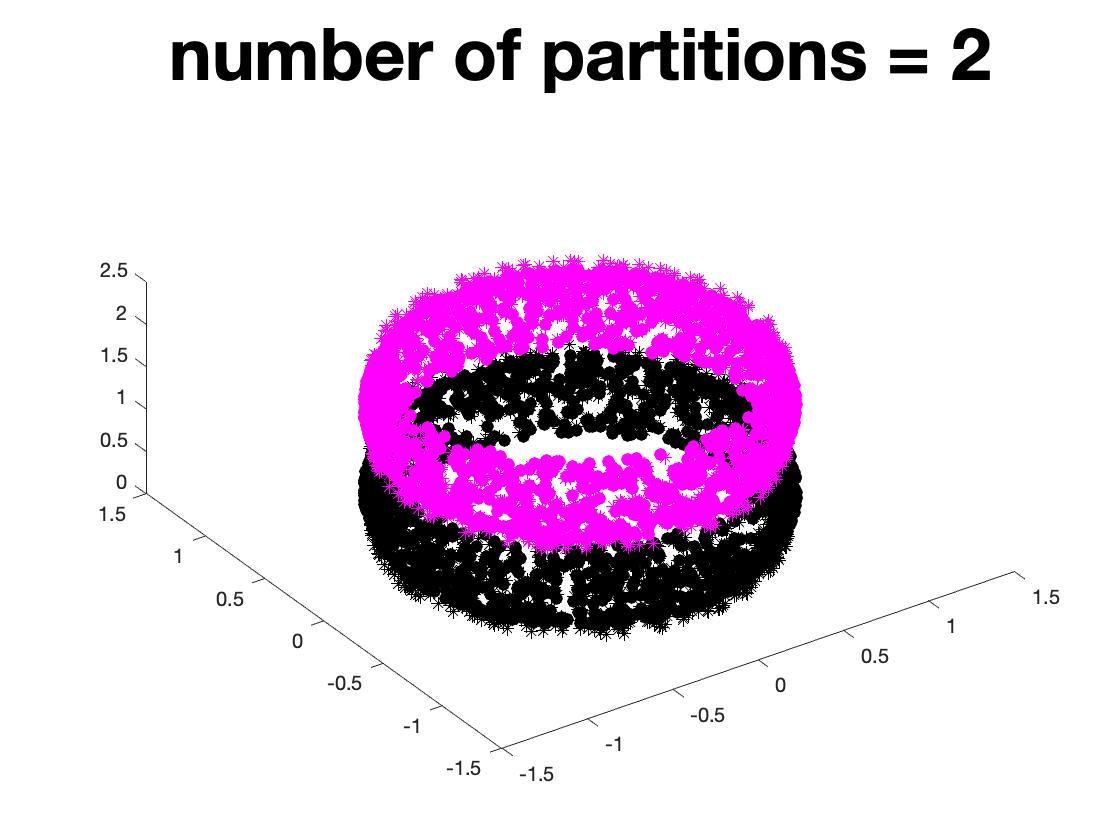}
	\includegraphics[height = 120pt,width=0.33\textwidth]{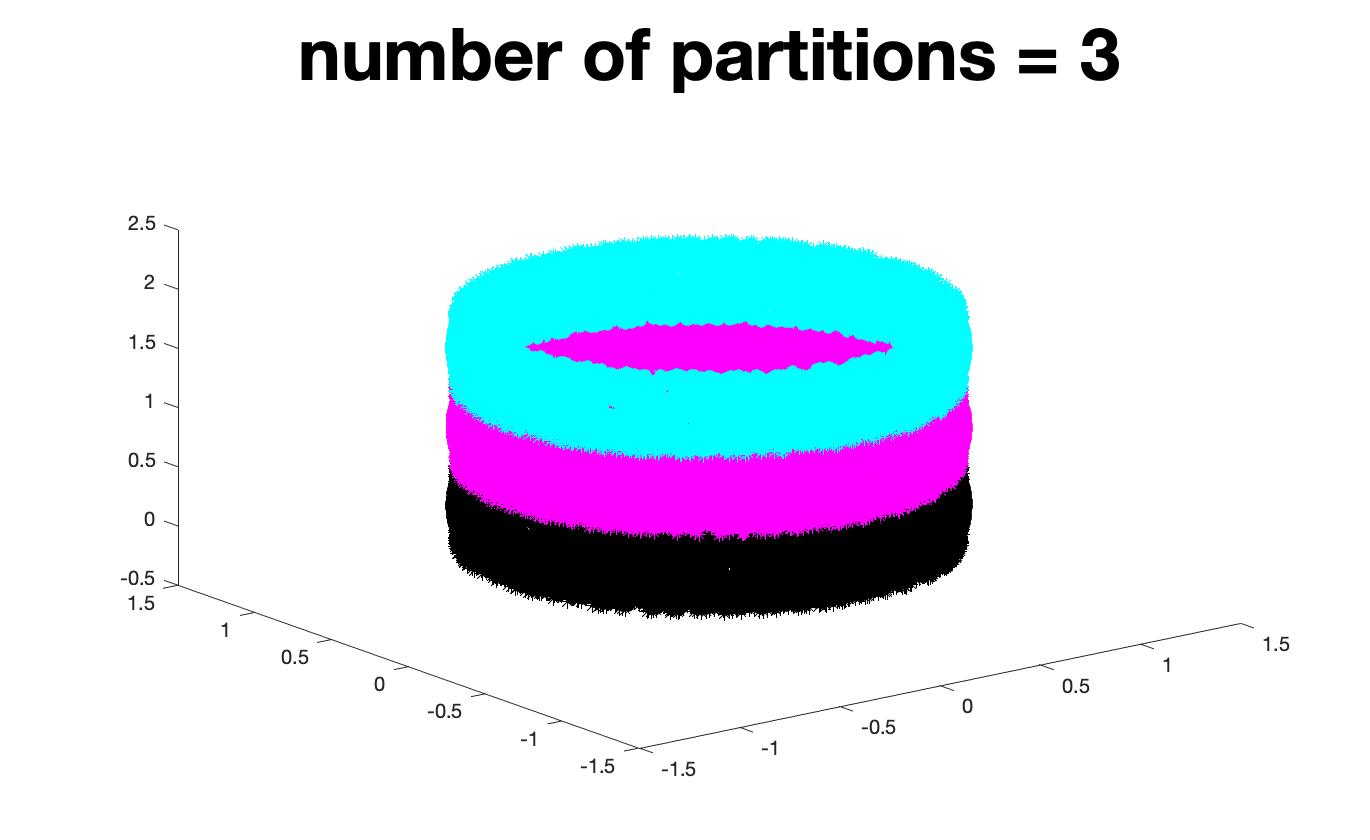}
	\caption{Spherelets on cylinder, left: convergence rate of MSE w.r.t. sample size; mid: projected samples to the estimated manifold with 2 spheres; right: projected samples to the estimated manifold with 3 spheres; colors represent different local neighborhoods.}
	\label{fig:rate_cylinder}
\end{figure}

\paragraph{Estimating the partition}

Next, we study the relation between number of partitions and MSE and compare with PCA. There are many existing partitioning algorithms for subdividing the sample space into local neighborhoods. Popular algorithms, such as cover trees, iterated PCA and METIS, have a multi-scale structure, repeatedly partitioning $\RR^D$ until a stopping condition is achieved. 
For simplicity, for local SPCA and PCA, we consider a multi-scale partitioning scheme which iteratively splits the sample space based on the first principal component score until a predefined bound $\epsilon$ on MSE is met for each of the partition sets or the sample size $n_k$ no longer exceeds a minimal value $n_0$. 
If $\mathrm{MSE}_k>\epsilon$ and $n_k>n_0$, we calculate $\mathrm{PC}_1=(X_{[k]}-\mu_k) v_{1,k}$, where $\mu_k=\bar{X}_{[k]}$ and $v_{1,k}$ is the first eigenvector of the covariance matrix of $X_{[k]}$. Next we split $C_k$ into two sub-partitions $C_{k,1}$ and $C_{k,2}$ based on the sign of $\mathrm{PC}_1$, i.e, $i^{th}$ sample of $X_{[k]}$ is assigned to $C_{k,1}$ if $\mathrm{PC}_{1,i}>0$ and to $C_{k,2}$ otherwise. We estimate the intrinsic dimension $d$ as corresponding to the elbow point in the dimension v.s. MSE plot for PCA or SPCA, following common practice in PCA. 
Estimating the intrinsic dimension $d$ is known to be a difficult problem for algorithms lacking 
fitted values $\widehat{X}_i$ for $X_i$; see \cite{levina2005maximum, lin2008riemannian, little2009estimation, facco2017estimating}.
\begin{figure}[!h]
	
	\centering
	\includegraphics[width=0.8\textwidth, height=150pt]{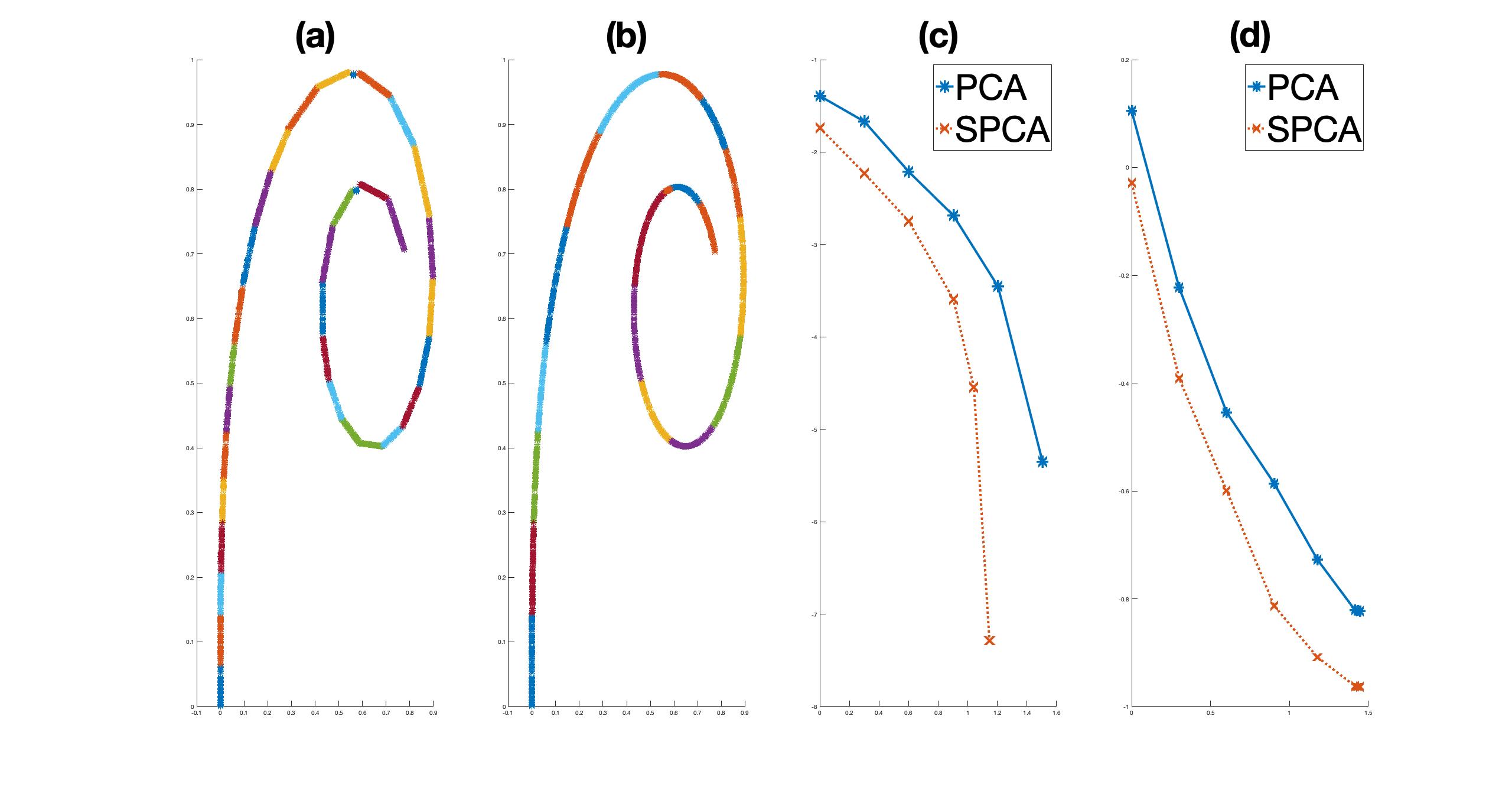}
	\caption{Projection of Euler spiral data by local PCA (a) and local SPCA (b) colored by partition; plot of log(MSE) vs. log(number of partitions) for Euler spiral (c) and Economics data (d).}
	\label{fig:spherelets_MSE}
\end{figure}

{\bf Euler Spiral}. We generate $2500$ training and $2500$ test samples from the Euler spiral with Gaussian noise with variance $0.01$. Figure \ref{fig:spherelets_MSE}(a) and (b) show the projected test dataset with different partitions described in different colors. It is clear that there are fewer pieces of circles than lines and the estimated $\widehat{M}$ is smoother in the second panel, reflecting better approximation by SPCA. Figure \ref{fig:spherelets_MSE}(c) shows the comparative performance of local PCA and SPCA with respect to log of number of partitions vs log of MSE. Clearly SPCA has much better performance than PCA, as it requires only 14 partitions to achieve an MSE of about $10^{-7}$, while PCA requires 120 partitions to achieve a similar error. 

{\bf Economics}. As introduced in Section \ref{sec:intro}, the Economics dataset has $D=5$ attributes and $576$ samples. We choose $460$ samples randomly as the training set and the remaining $116$ samples as the test set for cross validation. We choose $d=1$ as the first principal component explains $99.7\%$ of the variance. Figure \ref{fig:spherelets_MSE}(d) shows that SPCA has much better performance than PCA, as it requires only 8 partitions to achieve an MSE of $0.81$, while PCA requires 26 partitions to achieve a similar error.



{\bf User knowledge}. These data were collected to assess students’ knowledge of Electrical DC Machines \citep{kahraman2013development}. There were $n=258$ students who were studied to obtain $D=5$ attributes including study time for goal object materials, repetition number for goal object materials, study time for objects related to the goal object, exam performance for related objects and exam performance for the goal object. We choose $206$ samples randomly as the training set and the remaining $52$ samples as the test set for cross validation. We chose $d=4$ because we did not observe any sudden drop (also known as elbow point) in the MSE plot in Figure \ref{fig:spherelets_MSE_User} left panel. Figure  \ref{fig:spherelets_MSE_User} right panel compares performance of PCA and SPCA. Overfitting can occur as the number of partitions becomes too large, forcing the sample size per partition to be too low.  The optimal number of partitions for SPCA is $8$, which is significantly lower than the optimal number for PCA.  In addition, the MSE for spherelets with $8$ partitions is $0.25$, which is significantly lower than the minimal value obtained by local PCA.
\begin{figure}
	
	\centering
	\includegraphics[width=0.8\textwidth, height=120pt]{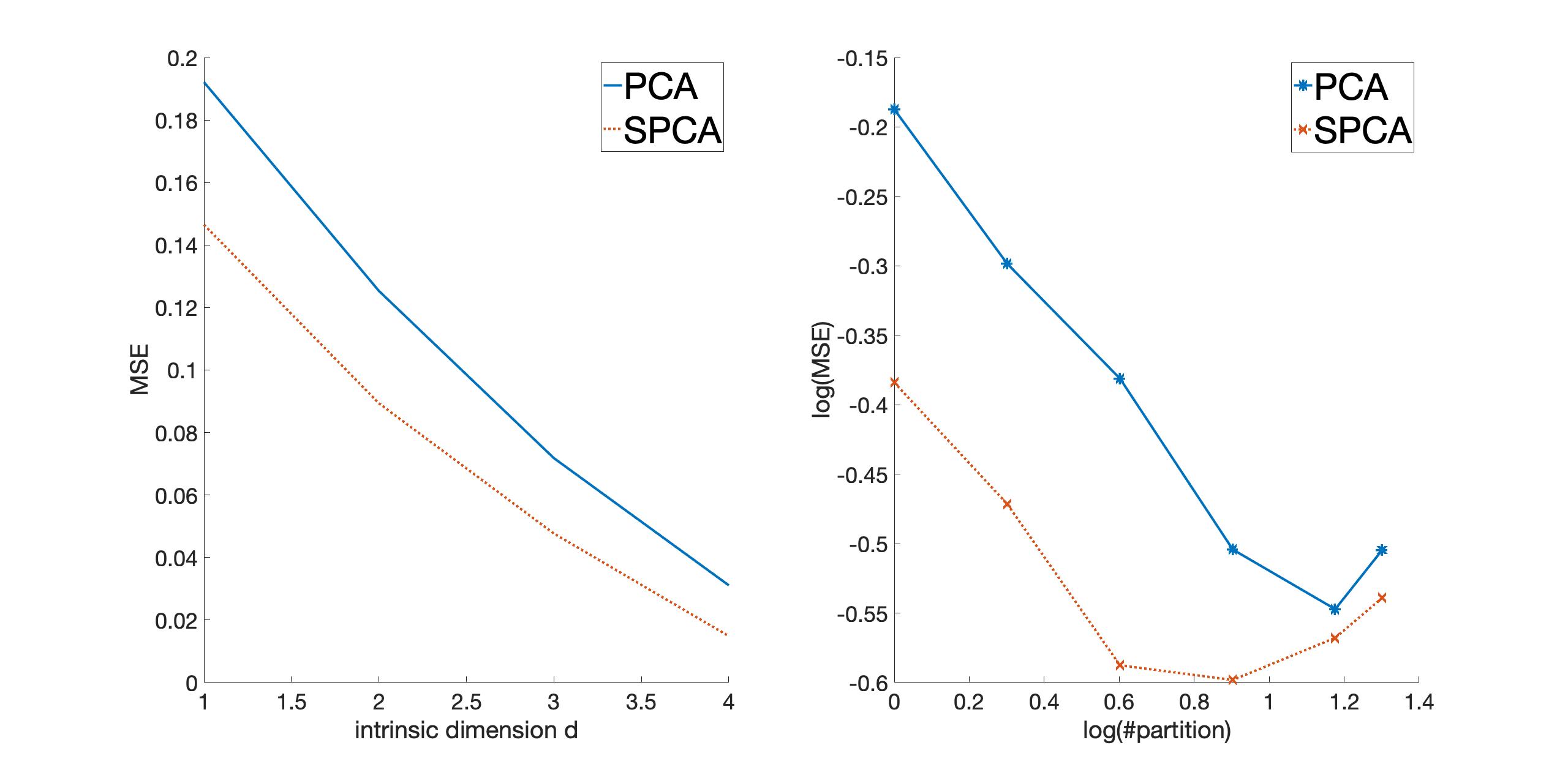}
	
	\caption{Left: MSE of local PCA and spherelets for different dimension; right: log(MSE) vs. log(number of partitions) for the user knowledge dataset}
	\label{fig:spherelets_MSE_User}
\end{figure}

The above examples show that spherelets has smaller MSE than local PCA given the same number of partitions. Equivalently, given a fixed error $\epsilon$, the number of spheres needed to approximate the manifold is smaller than that of hyperplanes. This coincides with the statement of Theorem \ref{mainthm}. When training sample size is small to moderate, there will be limited data available per piece and local PCA will have high error when the number of pieces is sufficiently large to obtain an accurate approximation to the manifold.

\section{Discussion}

There are several natural next directions building on the spherelets approach. The current version of spherelets is not constrained to be connected, so that the estimate $\widehat{M}$ of the manifold $M$ will in general be disconnected.  We view this as an advantage in many applications, because it avoids restricting consideration to manifolds that have only one connected component and instead accommodates true disconnectedness.  Nevertheless, in certain applications it is very useful to obtain a connected estimate; for example, when we have prior knowledge that the true manifold is connected and want to use this knowledge to improve statistical efficiency and produce a more visually appealing and realistic estimate. A typical case is in imaging when $D$ is 2 or 3 and $d$ is 1 or 2 and we are trying to estimate a known object from noisy data.  
Possibilities for incorporating connectedness constraints include (a) producing an initial $\widehat{M}$ using spherelets and then closing the gaps through linear interpolation; and (b) incorporating a continuity constraint directly into the objective function, to obtain essentially a type of higher dimensional analogue of splines. 


An additional direction is improving the flexibility of the basis by further broadening the dictionary beyond simply pieces of spheres.  Although one of the main advantages of spherelets is that we maintain much of the simplicity and computational tractability of locally linear bases, it is nonetheless intriguing to include additional flexibility in an attempt to obtain more concise representations of the data with fewer pieces.  Possibilities we are starting to consider include the use of quadratic forms to obtain a higher order local approximation to the manifold and extending spheres to ellipses.  In considering such extensions, there are major statistical and computational hurdles; the statistical challenge is maintaining parsimony, while the computational one is to obtain a simple and scalable algorithm.  As a good compromise to clear both of these hurdles, one possibility is to start with spherelets and then perturb initial sphere estimates (e.g., to produce an ellipse) to better fit the data.

Another important direction is to study the optimal partitioning for spherelets. Existing partitioning algorithms are locally linear, mainly designed for local PCA. A spherical partitioning algorithm needs to be developed to improve spherelets. With such partitioning, it might be possible to verify the covering number upper bound numerically.

Finally, there is substantial interest in scaling up to very large $D$ cases; the current algorithm will face problems in this regard similar to issues faced in applying usual PCA to high-dimensional data.  To scale up spherelets, one can potentially leverage on scalable extensions of PCA, such as sparse PCA (\cite{sparsePCA2006}, \cite{sparsePCA2009}).  The availability of a very simple closed form solution to spherical PCA makes such extensions conceptually straightforward, but it remains to implement such approaches in practice and carefully consider appropriate asymptotic theory.   In terms of theory, it is interesting to consider optimal rates of simultaneously estimating $M$ and the density of the data on (or close) to $M$, including in cases in which $D$ is large and potentially increasing with sample size.  

\section*{Acknowledgement}
DL and DD were supported by United States Office of Naval Research, N00014-14-1-0245 and N00014-16-1-2147; United States National Institutes of Health, 5R01ES027498-02. 

\section{Appendix}
\subsection{Economics data fit}
As the Economics data have clear evidence in favor of $d=1$, we also consider applying principal curves in addition to local PCA and local SPCA (spherelets).  We present the data and fitted values for representative pairs of variables in Figure \ref{fig:economics_all}. We find that principal curves over-smooths the data and does not have competitive performance relative to the local PCA-based methods.  In addition, spherelets clearly have the best performance for all pairs of features, with more gain when the data exhibit more curvature.
\begin{figure}[!h]
	\begin{center}
		\includegraphics[width=1\textwidth]{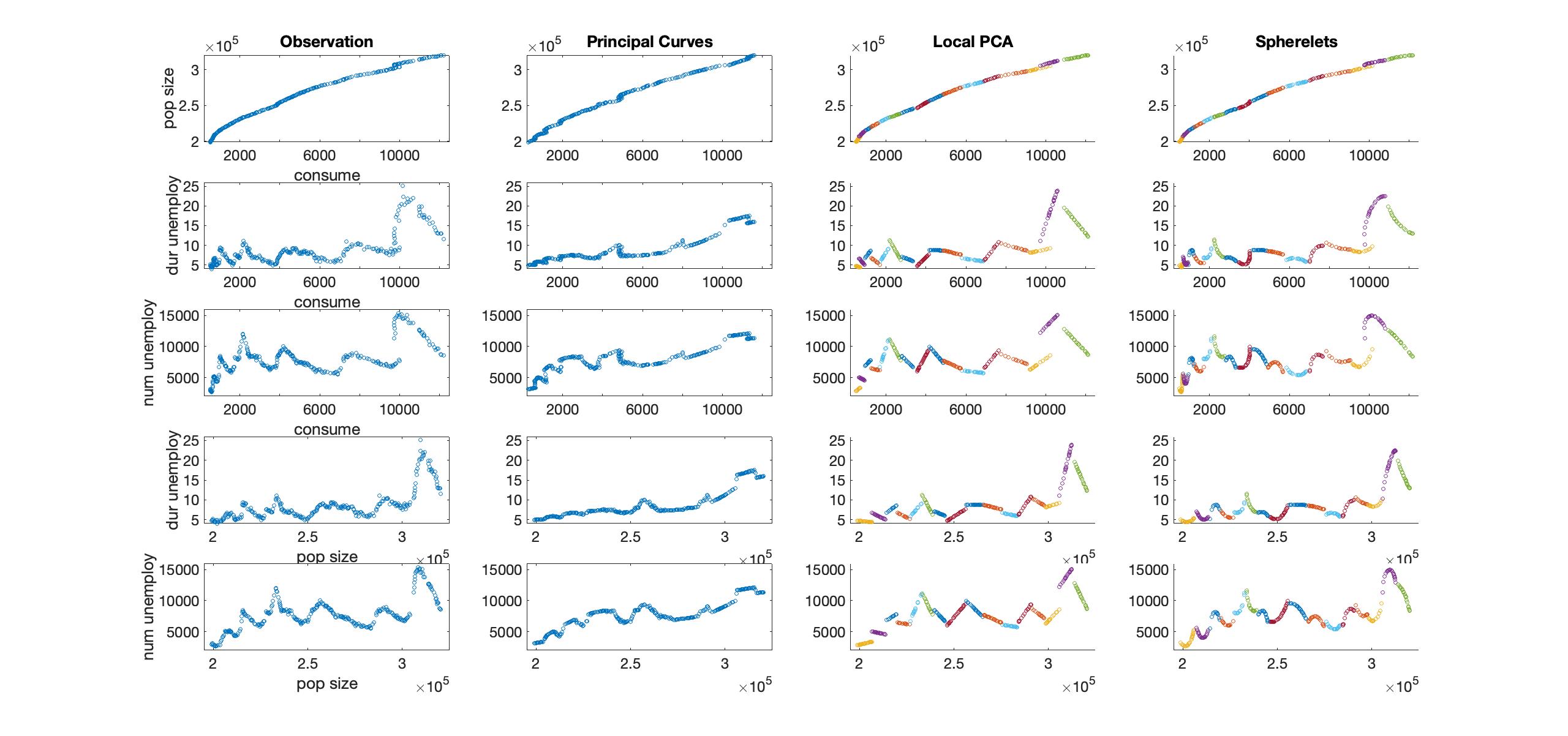}
		\caption{Illustration of results based on fitting principal curves, local PCA and local SPCA (spherelets) to the Economics data from Figure \ref{fig:economics}.}
		\label{fig:economics_all}
	\end{center}
\end{figure}

\subsection{SPCA solution}
In this section we prove the SPCA solution shown in  Section \ref{sec:SPCA}.
\begin{proof}[Proof of Lemma \ref{sproj}]
	Let $\Phi_{V,c}$ be the orthogonal projection to the affine subspace $c+V$; that is, $\Phi_{V,c}(x)=c+VV^\top (x-c)$. Then observe that $x-\Phi_{V,c}(x)\perp \Phi_{V,c}(x)-y$, $\forall y\in S_{V}(c,r)$, so 
	$$\|x-y\|^2=\|x-\Phi_{V,c}(x)+\Phi_{V,c}(x)-y\|^2=\|x-\Phi_{V,c}(x)\|^2+\|\Phi_{V,c}(x)-y\|^2.$$
	That is, the optimization problem $\underset{y\in S_V(c,r)}{\argmin}\|x-y\|^2$ is equivalent to $\underset{{y\in S_V(c,r)}}{\argmin}\|\Phi_{V,c}(x)-y\|^2$. Since the second problem only involves the affine subspace $c+V$, we can translate it to the following problem:
	$$\underset{y\in S(c,r)\subset\RR^{d+1}}{\argmin}\|x-y\|^2,$$
	where $x$ is any point in $\RR^{d+1}$ and $S(c,r)=\{y\in \RR^{d+1}:\  \|y-c\|=r\}$. So we only need to prove
	$$\Psi_{V,c}(x)\coloneqq \underset{y\in S(c,r)}{\argmin}\|x-y\|^2=c+\frac{r}{\|x-c\|}(x-c).$$
	On one hand,
	\begin{align*}
		\|x-\Psi_{V,c}(x)\|^2&=\left\|x-c-\frac{r}{\|x-c\|}(x-c)\right\|^2=\left\|\left(1-\frac{r}{\|x-c\|}\right)\left(x-c\right)\right\|^2\\
		&=\left(1-\frac{r}{\|x-c\|}\right)^2\|x-c\|^2=\left(\|x-c\|-r\right)^2.
	\end{align*}
	On the other hand, for any $y\in S(c,r)$,
	\begin{align*}
		\|x-y\|^2&=\|x-c+c-y\|^2=\|x-c\|^2+\|c-y\|^2-2(x-c)^\top(y-c)\\
		&=\|x-c\|^2+r^2-2(x-c)^\top (y-c)\geq \|x-c\|^2+r^2-2\|x-c\|\|y-c\|\\
		&=\|x-c\|^2+r^2-2r\|x-c\|=(\|x-c\|-r)^2=\|x-\Psi_{V,c}(x)\|^2. 
	\end{align*}
\end{proof}

\begin{proof}[Proof of Theorem \ref{SPCA}]
	By simple calculation, we can show that  
	$$\widehat{r}^2\coloneqq\underset{r}{\argmin }\ \mathscr{L}(c,r)=\frac{1}{n}\sum_{i=1}^n\|Y_i-c\|^2.$$
	Hence, if we adopt the same notation: 
	$$\mathscr{L}(c)\coloneqq \sum_{i=1}^n(\|Y_i-c\|^2-\widehat{r}^2)^2=\sum_{i=1}^n\Big\{ \|Y_i-c\|^2-\frac{1}{n}\sum_{j=1}^n\|Y_j-c\|^2)\Big\}^2,$$
	it suffices to minimize $\mathscr{L}(c)$ to obtain $\widehat{c}$.
	$$\mathscr{L}(c)=\sum_{i=1}^n\Big\{ Y_i^\top Y_i-2c^\top Y_i-\frac{1}{n}\sum_{j=1}^n(Y_j^\top Y_j-2c^\top Y_j)\Big\}^2.$$
	Letting $l_i=Y_i^\top Y_i$, $\displaystyle{\bar l=\frac{1}{n}\sum_{i=1}^n {l_i}}$ and $\displaystyle{\bar Y=\frac{1}{n}\sum_{i=1}^n {Y_i}}$, then
	\begin{align*}
		\mathscr{L}(f)&=\sum_{i=1}^n\Big\{ Y_i^\top Y_i-2c^\top Y_i-\frac{1}{n}\sum_{j=1}^n(Y_j^\top Y_j-2c^\top Y_j)\Big\}^2\\
		&=\sum_{i=1}^n\Big(l_i-2c^\top Y_i-\bar l+2c^\top \bar Y\Big)^2=\sum_{i=1}^n\Big(   (l_i-\bar l)-2c^\top (Y_i-\bar Y)\Big)^2\\
		&=\sum_{i=1}^n \bigg\{4c^\top (Y_i-\bar Y)(Y_i-\bar Y)^\top c-4(l_i-\bar l)c^\top (Y_i-\bar Y)+(l_i-\bar l)^2\bigg\}\\
		&=4c^\top \sum_{i=1}^n(Y_i-\bar Y)(Y_i-\bar Y)^\top c-4c^\top \sum_{i=1}^n(l_i-\bar l)(Y_i-\bar Y)+\sum_{i=1}^n(l_i-\bar l)^2
	\end{align*}
	is a quadratic function. So $\widehat c=\frac{1}{2}H^{-1}\xi$ where
	$$H=\sum_{i=1}^n\big(Y_i-\bar Y)(Y_i-\bar Y\big)^\top,\quad 
	\xi=\sum_{i=1}^n \bigg(Y_i^\top  Y_i-\frac{1}{n}\sum_{j=1}^nY_j^\top Y_j\bigg)\big(Y_i-\bar Y\big).$$
\end{proof}
Then we prove Corollary \ref{lossfunctions}, that is, when the data are sampled from some sphere $S_{V_0}(c_0,r_0)$, then SPCA and the sum of squared residuals have the same minimizer:
$$(V_0,c_0,r_0)=(\widehat{V},\widehat c,\widehat{r})
=\underset{V,c,r}\argmin\ \sum_{i=1}^n
d^2(Y_i,S_V(c,r))\eqqcolon\underset{V,c,r}\argmin\ \widetilde{\mathscr{L}}(c). $$

\begin{proof}[Proof of Corollary \ref{lossfunctions}]
	
	When $X_i\in S_{V_0}(c_0,r_0)$, both geometric and algebraic loss functions are zero at $(V_0,c_0,r_0)$, which is the minimizer of both algorithms.

\end{proof}	
\subsection{A deeper look at Theorem 8} \label{sec:Thm8}
We define some key geometric features of the manifold $M$ involved in the covering number proof.
\begin{definition}\label{UTM}
	Let $M$ be a $d$-dimensional $C^3$ Riemannian manifold. The volume of $M$ is denoted by $V$. Let $\kappa:T^1M\rightarrow \RR$ denote the curvature of the geodesic on $M$ starting from point $p$ with initial direction $v$:
	$$\kappa(p,v):=\left\|\frac{d^2 \exp_p(tv)}{dt^2}\bigg|_{t=0}\right\|,$$
	where $T^1M\coloneqq \bigcup_{p\in M}\left\{v\in T_pM\mid \|v\|=1\right\}$ is the unit sphere bundle over $M$ and $\exp_p(\cdot)$ is the exponential map from the tangent plane at $p$ to $M$. 
	Then $\displaystyle{\kappa_{\max}\coloneqq \sup_{(p,v)\in T^1M}|\kappa(p,v)|<\infty}$ is the maximum curvature. Similarly,
	$$T\coloneqq \sup_{(p,v)\in T^1M} \left\|\frac{d^3 \exp_p(tv)}{dt^3}\bigg|_{t=0}\right\|$$
	is the maximum of the absolute rate of change of the curvature. 
\end{definition}
\begin{definition}
	Given any $\epsilon>0$ and letting $T_pM$ denote the tangent plane to $M$ at $p \in M$, 
	$$F_\epsilon\coloneqq\left\{p\in M: \sup_{v\in T^1_pM}\kappa(p,v)-\inf_{v\in T^1_pM}\kappa(p,v) \leq \left(\frac{2\epsilon}{\kappa_{\max}}\right)^\frac{1}{2}\right\}$$
	is called the set of $\epsilon$-spherical points on $M$, where $T^1_pM$ is the unit ball in tangent space $T_pM$. Let $B(p,\epsilon)$ be the geodesic (open) ball centered at $p$ with radius $\epsilon$, then $\displaystyle{M_\epsilon:=\bigcup_{p\in F_\epsilon} B\left(p,\frac{1}{2}\left(\frac{6\epsilon}{3+T}\right)^\frac{1}{3}\right)}$	is called the spherical submanifold of $M$, and the volume is $V_{\epsilon}\coloneqq \mathrm{Vol}(M_\epsilon)$. $M$ is called an $\epsilon$ sphere if $V_{\epsilon}=V$.
\end{definition}

{$V_\epsilon$ is non-decreasing with respect to $\epsilon$}, and it is possible that $V_\epsilon=V$ but $M_\epsilon\neq M$. In this case, $M\setminus M_\epsilon$ has zero Riemannian measure so it does not impact the manifold approximation given observations from the manifold. As a result, we will not consider zero measure sets in the following sections.  \begin{example}
	A space form, a complete, simply connected Riemmanian manifold of constant sectional curvature, is an $\epsilon$ sphere for any manifold dimension $d$ and $\epsilon>0$.
\end{example}
\begin{example}
	A one dimensional manifold (a curve) is an $\epsilon$ sphere for any $\epsilon>0$.
\end{example}


Now we present a nontrivial surface, called the Enneper's surface (a minimal surface) and calculate the spherical points and spherical submanifold explicitly. 

\begin{example}
	Let $M=\left\{(u-\frac{1}{3}u^3+uv^2,-v-u^2v+\frac{1}{3}v^3,u^2-v^2)\in\RR^3|u^2+v^2\leq R^2\right\}$ be the compact truncation of the Enneper surface, which is an interesting surface in differential geometry that has varying curvature. In fact this surface is a minimal surface, that is, the mean curvature is zero and the two principal curvatures are mutually additive inverse to each other. By definition, $M$ is a compact smooth surface. We calculate the spherical points, spherical submanifold as well as its volume. 
	\begin{proposition}\label{Enneper}For the Enneper surface, when $\epsilon <\frac{4}{(R^2+1)^2}$, there are no spherical point so $F_\epsilon=M_\epsilon=\emptyset$ and $V_\epsilon=0$. When $ \epsilon \geq \frac{4}{(R^2+1)^2}$,	
		$$F_\epsilon=\left\{(u-\frac{1}{3}u^3+uv^2,-v-u^2v+\frac{1}{3}v^3,u^2-v^2):\frac{2}{\sqrt{\epsilon}}-1\leq u^2+v^2\leq R^2\right\},$$
		and $T= 
		\begin{cases}
			\frac{1}{1024\sqrt{7}} & R\geq \frac{1}{\sqrt{7}} \\ \frac{4R}{(R^2+1)^4} & R<\frac{1}{\sqrt{7}}.
		\end{cases}$
		When $R\geq\frac{1}{\sqrt 7}$,  $$M_{\epsilon}=\Big\{(u-\frac{1}{3}u^3+uv^2,-v-u^2v+\frac{1}{3}v^3,u^2-v^2)|u^2+v^2\geq\alpha^2\Big\},$$
		where $\alpha=\max\left\{0,\sqrt{\frac{2}{\sqrt\epsilon}-1}-\left(\frac{6\epsilon}{3+\frac{1}{1024\sqrt7}}\right)^{\frac{1}{3}}\right\}$, then
		$V_{\epsilon}=\pi\Big(R^2+\frac{1}{2}R^4-\alpha^2-\frac{1}{2}\alpha^4\Big), 
		\frac{V_{\epsilon}}{V}=\frac{2R^2+R^4-2\alpha^2-\alpha^4}{2R^2+R^4}.$ 
		An extreme case is $\epsilon\geq \frac{4}{(R^2+1)^2}$ and $\sqrt{\frac{2}{\sqrt\epsilon}-1}-\left(\frac{6\epsilon}{3+\frac{1}{1024\sqrt7}}\right)^{\frac{1}{3}}\leq0$. In this case, although $F_\epsilon\neq M$, the union of small geodesic balls centered on the spherical points is $M$, that is, $M_\epsilon=M$, so $\alpha=0$ and $\frac{V_{\epsilon}}{V}=1$, which means $M$ is an $\epsilon$ sphere.

	\end{proposition}
	\begin{proof}
		The first fundamental form of Enneper's surface is $E=(1+u^2+v^2)^2$, $F = 0$, $G = (1+u^2+v^2)^2$, the two principal curvatures are $k_1 = \frac{2}{(1+u^2+v^2)^2},k_2 = -\frac{2}{(1+u^2+v^2)^2},k_2$. So the Gaussian curvature is $-\frac{4}{(1+u^2+v^2)^4}$ and the mean curvature is $0$ so $K=4$. 
		Then all other quantities including $F_\epsilon$ and $M_\epsilon$ can be directly calculated, see \cite{fischer2017mathematical,weisstein2009enneper} for more details.	    
	\end{proof}
\end{example}

\bibliographystyle{chicago} 
\bibliography{sample}

\newpage
\section*{Supplementary Materials}
\section{SPCA Asymptotics}\label{sec:SPCAasym}
In this section, we prove the SPCA asymptotic results from Section 2.2. 
\begin{proof}[Proof of Theorem 4]
	Let $V^*,c^*,r^*$ be the population SPCA solution, then the error term splits into bias and variance:
	$$\|c_0-\widehat{c}_n\|\leq \underbrace{\|c_0-c^*\|}_{\text{Bias}}+\underbrace{\|c^*-\widehat{c}_n\|}_{\text{Variance}}.$$
	It suffices to bound the bias and variance separately. We handle the bias term in Lemma \ref{lem:SPCAbias} and variance in Lemma \ref{lem:SPCAvar}, and then Theorem 4 follows.
\end{proof}

\begin{lemma}\label{lem:SPCAbias}
	Let $V^*,c^*,r^*$ be the population SPCA solution, then $V^*=V_0$,
	$$\|c^*-c_0\|\leq \frac{\sigma^2}{2\widetilde{\lambda}_{d+1}(\widetilde{\lambda}_{d+1}+\sigma^2)}\|\xi_{Y}\|+\frac{\sigma^2}{2(\widetilde{\lambda}_{d+1}+\sigma^2)}\| V_0V_0^\top\EE Y\|,$$
	$$|{r^*}^2-{r_0}^2|\leq 2\|\EE Y-c_0\|\|c_0-c^*\|+\|c_0-c^*\|^2+(d+1)\sigma^2,$$
	where $\xi_Y = \EE\left[(Y^\top Y-\EE(Y^\top Y))(Y-\EE Y)\right]$ and $\widetilde{\lambda}_{d+1}$ is the $d+1$th eigenvalue of the covariance matrix of $Y$.
\end{lemma}
\begin{proof}
	Let $\Sigma_X=\cov(X)$ and $\Sigma_Y=\cov(Y)$. First observe that $$\Sigma_Y = [v_1,\cdots,v_D]\diag\{\widetilde{\lambda}_1,\cdots,\widetilde{\lambda}_{d+1},0,\cdots,0\}[v_1,\cdots,v_D]^\top$$ with eigenvectors $v_1,\cdots,v_D$ since $Y$ is supported on $V$, a $d+1$-dimensional subspace. It is  clear that $V=\mathrm{span}\{v_1,\cdots,v_{d+1}\}$. Recall that $\Sigma_X = \Sigma_Y+\sigma^2\Id_D$, then $\Sigma_X =[v_1,\cdots,v_D] \diag\{\widetilde{\lambda}_1+\sigma^2,\cdots,\widetilde{\lambda}_{d+1}+\sigma^2,\sigma^2,\cdots,\sigma^2\}[v_1,\cdots,v_D]^\top$ and the eigenvectors of $X$ are also $v_1,\cdots,v_D$. Recall that $V^*$ consists of the first $d+1$ eigenvectors of $\Sigma_Y$, so $V^*=V_0$. 
	
	Let $\widehat{X}=\EE X+ V^*{V^*}^\top(X-\EE X)$ and $\widehat{Y}=\EE Y+ V_0{V_0}^\top(Y-\EE Y)$ be the projection of $X$ and $Y$ to the linear space $V^*=V_0$, then $\widehat{X}=\widehat{Y}+\widehat{\epsilon}$ where $\widehat{\epsilon}=V_0V_0^\top \epsilon$. The projected noise $\widehat{\epsilon}\sim N(0,\sigma^2 V_0V_0^\top)$. 
	The population solution of $SPCA$ is then given by 
	\begin{equation}\label{eqn:c*}
		c^* = -\frac{1}{2}\EE\left[(\widehat{X}-\EE \widehat{X})(\widehat{X}-\EE \widehat{X})^\top\right]^{-1
		}\EE\left[(\widehat{X}^\top \widehat{X}-\EE(\widehat{X}^\top \widehat{X}))(\widehat{X}-\EE \widehat{X})\right].
	\end{equation}
	Recall that $c^*=\underset{c}{\argmin} \var(\|X-c\|^2)$ and $c_0=\underset{c}{\argmin}\var(\|Y-c\|^2)$ since $\var(\|Y-c_0\|^2)=0$ a.s. Letting $\sigma^2\to 0$, all moments of $X$ converge to the corresponding moments of $Y$, so $\var(\|X-c\|^2)\xrightarrow[]{\sigma^2\to0}\var(\|Y-c_0\|^2)$. By convexity of the objective function, the minimizer is unique so $c^*\to c_0$ as $\sigma^2\to 0$ and Equation (\ref{eqn:c*}) becomes 
	$$c_0 = -\frac{1}{2}\EE\left[(\widehat{Y}-\EE \widehat{Y})(\widehat{Y}-\EE \widehat{Y})^\top\right]^{-1
	}\EE\left[(\widehat{Y}^\top \widehat{Y}-\EE(\widehat{Y}^\top \widehat{Y}))(\widehat{Y}-\EE \widehat{Y})\right].$$
	First observe that 
	$\widehat{Y}=Y$, $\EE \widehat{X}=\EE X = \EE Y = \EE\widehat{Y},$
	$\cov(\widehat{Y})=\Sigma_{\widehat{Y}}=\Sigma_Y$,$\cov(\widehat{X})=\Sigma_{\widehat{X}}=V^*{V^*}^\top \Sigma_X V^*{V^*}^\top=\Sigma_Y+\sigma^2V_0V_0^\top.$
	Then observe that
	\begin{align*}
		\xi_{\widehat{X}} & = \EE\left[(\widehat{X}^\top \widehat{X}-\EE(\widehat{X}^\top \widehat{X}))(\widehat{X}-\EE \widehat{X})\right]\\
		& = \EE\left[(\widehat{Y}+\widehat{\epsilon})^\top (\widehat{Y}+\widehat{\epsilon})-\EE((\widehat{Y}+\widehat{\epsilon})^\top (\widehat{Y}+\widehat{\epsilon})))(\widehat{Y}+\widehat{\epsilon}-\EE (\widehat{Y}+\widehat{\epsilon}))\right]\\
		& = \EE\left[\left(\widehat{Y}^\top \widehat{Y}-\EE \widehat{Y}^\top \widehat{Y}+\widehat{\epsilon}^\top \widehat{\epsilon}-\EE\widehat{\epsilon}^\top\widehat{\epsilon}+2\widehat{\epsilon}^\top \widehat{Y}\right)\left(\widehat{Y}-\EE \widehat{Y}+\widehat{\epsilon}\right)\right]\\
		& =\EE\left[\left(\widehat{Y}^\top \widehat{Y}-\EE(\widehat{Y}^\top \widehat{Y})\right)\left(\widehat{Y}-\EE \widehat{Y}\right)\right]+\EE\left[\left(\widehat{Y}^\top \widehat{Y}-\EE \widehat{Y}^\top \widehat{Y}\right)\widehat{\epsilon}\right]\\
		& ~~~+ \EE\left[\left(\widehat{\epsilon}^\top \widehat{\epsilon}-\EE\widehat{\epsilon}^\top\widehat{\epsilon}+2\widehat{\epsilon}^\top \widehat{Y}\right)(\widehat{Y}-\EE \widehat{Y})\right]+ \EE\left[\left(\widehat{\epsilon}^\top \widehat{\epsilon}-\EE\widehat{\epsilon}^\top\widehat{\epsilon}+2\widehat{\epsilon}^\top \widehat{Y}\right)\widehat{\epsilon}\right]\\
		& = \xi_{\widehat{Y}} +2\EE\left[\widehat{Y}^\top\widehat{\epsilon} \widehat{\epsilon}\right]= \xi_{\widehat{Y}}+2\sigma^2V_0V_0^\top\EE Y.
	\end{align*}
	Let $U=[v_1,\cdots,v_D]$, then $\Sigma_Y = U\diag\{\widetilde{\lambda}_1,\cdots,\widetilde{\lambda}_{d+1},0\cdots,0\}U^\top$. Now we can compare $c^*$ and $c_0$:
	\begin{align*}
		&\|c^*-c_0\| = \left\|-\frac{1}{2}\Sigma_{\widehat{X}}^{-1}\xi_{\widehat{X}}+\frac{1}{2}\Sigma_{\widehat{Y}}^{-1}\xi_{\widehat{Y}}\right\|\\
		& = \frac{1}{2}\left\| \left(\Sigma_Y+\sigma^2V_0V_0^\top\right)^{-1}\left(\xi_{\widehat{Y}}+\sigma^2V_0V_0^\top\EE Y\right)-\Sigma_Y^{-1}\xi_{\widehat{Y}}\right\|\\
		& = \frac{1}{2}\left\|U\diag\left\{\frac{1}{\widetilde{\lambda}_1+\sigma^2},\cdots,\frac{1}{\widetilde{\lambda}_{d+1}+\sigma^2},0,\cdots,0\right\}U^\top\xi_{\widehat{Y}}-\Sigma_Y^{-1}\xi_{\widehat{Y}}+\left(\Sigma_Y+\sigma^2V_0V_0^\top\right)^{-1}\sigma^2V_0V_0^\top\EE Y\right\|\\
		& \leq \frac{1}{2}\left\|U\diag\left\{\frac{-\sigma^2}{\widetilde{\lambda}_1(\widetilde{\lambda}_1+\sigma^2)},\cdots,\frac{-\sigma^2}{\widetilde{\lambda}_{d+1}(\widetilde{\lambda}_{d+1}+\sigma^2)},0,\cdots,0\right\}U^\top\xi_{\widehat{Y}}\right\|\\
		&\hspace{0.05cm}+\frac{1}{2}\left\|U\diag\left\{\frac{\sigma^2}{\widetilde{\lambda}_1+\sigma^2},\cdots,\frac{\sigma^2}{\widetilde{\lambda}_{d+1}+\sigma^2},0,\cdots,0\right\}U^\top V_0V_0^\top\EE Y\right\|\\
		&\leq \frac{\sigma^2}{2\widetilde{\lambda}_{d+1}(\widetilde{\lambda}_{d+1}+\sigma^2)}\|\xi_{Y}\|+\frac{\sigma^2}{2(\widetilde{\lambda}_{d+1}+\sigma^2)}\| V_0V_0^\top\EE Y\|.
	\end{align*}
	Then we consider the radius:
	\begin{align*}
		\left|{r^*}^2-r_0^2\right| & = \left|\EE[\|\widehat{X}-c^*\|^2]-\EE[\|Y-c_0\|^2]\right|\\
		& = \left|\EE\left[\|Y+\widehat{\epsilon}-c^*\|^2]-\EE[\|Y-c_0\|^2\right]\right|\\
		& = \left|\EE\|Y-c_0+c_0-c^*+\widehat{\epsilon}\|^2-\EE\|Y-c_0\|^2\right|\\
		& = \left|\EE\left[2(Y-c_0)^\top(c_0-c^*+\widehat{\epsilon})+\|c_0-c^*\|^2+\|\widehat{\epsilon}\|^2+2(c_0-c^*)^\top\widehat{\epsilon}\right]\right|\\
		& = \left|2(\EE Y-c_0)^\top (c_0-c^*)+\|c_0-c^*\|^2+(d+1)\sigma^2\right|\\
		& \leq 2\|\EE Y-c_0\|\|c_0-c^*\|+\|c_0-c^*\|^2+(d+1)\sigma^2,
	\end{align*}
\end{proof}

\begin{lemma}\label{lem:SPCAvar}
	Let $v_i$ be eigenvectors of $\Sigma$ and $\widehat{v}_i$ be the corresponding eigenvectors of $\widehat{\Sigma}$, then under assumption (A), 
	$$\|\widehat{\bfc}_{n} -\bfc \|=o_p\left( \sigma n^{-1/2} \log n\right), ~~ \left|\widehat{r}_{n}-r \right|=o_p\left( \sigma n^{-1/2} \log n\right),~~\|\widehat{v}_i-v_i\|=o_p(n^{-1/2}\sigma\log n),~\forall i.$$
\end{lemma}

\begin{proof}
	Under assumption (A) \cite{Anderson_1963} showed asymptotic normality of the first $d+1$ eigenvalues and corresponding eigenvectors, when the underlying distribution is normal. \cite{Davis_1977} extended that result to general populations with finite fourth moment. Although the explicit forms of the joint distribution of the eigenvalues or eigenvectors are not available  \citep{Davis_1977}, the asymptotic distributions of  eigenvalues with multiplicities one, and the corresponding eigenvectors, are explicitly obtained. 
	
	Under assumption (A), let $\widehat{\lambda}_{j}$ be the $j$-th largest eigenvalue of the sample covariance matrix, for $j=1,\ldots,d+1$, then $\sqrt{n} \left( \widehat{\lambda}_{j}-\lambda_{j}\right)$ asymptotically follows a normal distribution with mean $0$ and variance $\kappa_{jjjj}+2\lambda_{j}^{2}$, where $\kappa_{jjjj}$ depends on the fourth order moment of $W=\Gamma^{T}(X-E(X)),$ where $\Gamma$ is the full population loading matrix. 
	
	Further, if $\widehat{v}_{j}$ is the sample eigenvector corresponding to the $j$-th largest eigenvalue, and $v_{j}$ is the $j$-th population eigenvector,  $\sqrt{n} \left(\widehat{v}_{j}-v_{j} \right)$ asymptotically follows a normal distribution with mean $0$ and variance $ \Xi_{j}$, where $\Xi_{j}$ is a positive definite matrix with finite components depending on the fourth cumulant of $W$.
	
	Finally, if $\sqrt{n}(x-\mu)$ is asymptotically normal with mean $0$ and covariance matrix $\Sigma$ with finite components, then it is easy to see that $\sqrt{n}\| x-\mu\|/(\log n)=o_{p}(1)$. 
	
	The remainder of the proof is split into three sections.
	\begin{enumerate}[label=\Roman*]
		\item {\it We first show that $\sqrt{n}(\sigma\log n)^{-1}\left(\cnh - \bfc \right) \xrightarrow{p} 0$.} The proof is split into three sub-parts.
		
		\begin{enumerate}[label = \roman*]
			\item {\it Showing $\sqrt{n}(\sigma\log n)^{-1}\| \widehat{H}^{+} - H^{+}\| \xrightarrow{p} 0$, with the spectral norm defined as $\|A\|=\sup\{\|Ax\|_{2}: \|x\|=1 \} $.} 
			
			Let $U=(V~ W)$, $\Lambda=\diag(\Lambda_{0}, \Lambda_{1})$. Then $H=VV^T\Sigma VV^T=V\Lambda_{0} V^T$. Similarly, $\widehat{H}=\vnh \widehat{\Sigma}_{d+1} \vnh^T$, where $\widehat{\Sigma}_{d+1}$ is the vector of first $(d+1)$ largest eigenvalues of $\widehat{\Sigma}$. Further, from the properties of Moore-Penrose inverse, we have $H^{+}=V\Lambda^{-1}_{0}V^{T}$, and similarly $\widehat{H}^{+}=\vnh \widehat{\Lambda}^{-1} \vnh^T$.
			
			Let $\|\cdot \|_F$ denote the Frobenius norm. Note that by triangle inequality
			\begin{equation}\label{eq_1}
				\|\vnh \widehat{\Lambda}^{-1} \vnh^T - V\Lambda^{-1}_{0}V^{T}\| 
				\leq  \|(\vnh-V) \widehat{\Lambda}^{-1} \vnh^T \| + \|V (\widehat{\Lambda}^{-1}-\Lambda^{-1}_{0}) \vnh^T \| + \| V\Lambda_{0}^{-1}(\vnh-V)^{T} \|.  
			\end{equation}
			The first term in right hand side (RHS) of (\ref{eq_1}), $\|(\vnh-V) \widehat{\Lambda}^{-1} \vnh^T \|=\|\vnh \widehat{\Lambda}^{-1} (\vnh-V)^{T}   \|=\| \widehat{\Lambda}^{-1} (\vnh-V)^{T}   \| \leq \|(\vnh-V)\widehat{\Lambda}^{-1} \|_F^2 \leq \widehat{\lambda}_{d+1}^{-1} \|(\vnh-V) \|_F^{2}$ where $\widehat{\lambda}_{d+1}$ is the $(d+1)^{th}$ largest eigenvalue of $\widehat{\Lambda}$, as $\vnh^T\vnh=I_{d+1}$.
			Further, $\sqrt{n} (\sigma \log n)^{-1} \|\vnh-V \|_F= \sqrt{ n (\sigma\log n)^{-2} \sum_{i=1}^{d+1}  \left\|\widehat{\bfv}_{i}-\bfv_{i} \right\|^2 } \xrightarrow{p} 0. $ 
			As $\lambda_{d+1}$ is bounded away from zero, by (A), $\widehat{\lambda}_{d+1}^{-1}=O_p(1)$. Thus, the first part of (\ref{eq_1}) converges to zero in probability after multiplying by $\sqrt{n}(\sigma \log n)^{-1}$.
			
			Consider the second term in RHS of (\ref{eq_1}). We have $$\|V(\widehat{\Lambda}^{-1}-\Lambda^{-1})\vnh^T\|\leq \|V(\widehat{\Lambda}^{-1}-\Lambda^{-1})\vnh^T\|_F=\|\widehat{\Lambda}^{-1}-\Lambda^{-1}\|_F
			,$$ as $V^TV=\vnh^T\vnh=I_{d+1}.$ 
			By the delta method, we have $\sqrt{n} \left\{ \left(\widehat{\lambda}_{d+1}\right)^{-1}-\left(\lambda_{d+1}\right)^{-1} \right\}=O_{p}(1)$.
			Thus, the second part of RHS of (\ref{eq_1}) converges to zero in probability after multiplying by $\sqrt{n}(\sigma\log n)^{-1}$.
			Finally, following a similar argument as before, it can be shown that the third part of RHS of (\ref{eq_1}) converges to zero in probability after multiplying by $\sqrt{n}(\sigma\log n)^{-1}$.
			
			
			\item {\it We next show that $\sqrt{n} \left(\sigma \log n \right)^{-1}\| \widehat{\xi} - { \xi}\|\xrightarrow{p} 0$.}
			
			As $E(\bfx_i)={0}$, we have 
			${ \xi}=E\left(\| V^T \bfx \|^2  VV^T \bfx   \right) - E \left( \| V^T \bfx \|^2  \right) E\left(  VV^T \bfx   \right).$
			First we will show the following three convergences in order:
			\begin{eqnarray}
				\sqrt{n} \left(\sigma \log n \right)^{-1} \left| \frac{1}{n} \sum_{i=1}^{n} \|\vnh^T \bfx_i \|^2  - E\left(\| V^T \bfx \|^2 \right) \right| &\xrightarrow{p}& 0\label{eq_2}\\ 
				\sqrt{n} \left(\sigma \log n \right)^{-1} \left| \frac{1}{n} \sum_{i=1}^{n} \vnh \vnh^T \bfx_i   - E\left(V V^T \bfx \right) \right| &\xrightarrow{p}& 0 \label{eq_3}\\ 
				\sqrt{n} \left( \sigma \log n \right)^{-1} \left|  \frac{1}{n} \sum_{i=1}^{n} \|\vnh^T \bfx_i \|^2 \vnh \vnh^T \bfx_i  - E\left(\| V^T \bfx \|^2  VV^T \bfx   \right) \right| &\xrightarrow{p}& 0\label{eq_4}
			\end{eqnarray} 
			
			The convergence in (\ref{eq_2}) is proved in two steps:
			\begin{eqnarray}
				\sqrt{n} \left(\sigma \log n \right)^{-1} \left|\frac{1}{n} \sum_{i=1}^{n} \|\vnh^T \bfx_i \|^2 -\frac{1}{n} \sum_{i=1}^{n} \|V^{T} \bfx_i \|^2  \right| \xrightarrow{p} 0	\label{eq_5}
			\end{eqnarray} 
			\begin{eqnarray}
				\sqrt{n} \left( \sigma \log n \right)^{-1} \left| \frac{1}{n} \sum_{i=1}^{n} \|V^T \bfx_i \|^2-  E\left(\| V^T \bfx \|^2 \right)\right| &\xrightarrow{p}& 0\label{eq_6}
			\end{eqnarray} 
			To see (\ref{eq_5}), observe that the left hand side (LHS) of (\ref{eq_5}) is proportional to 
			\begin{eqnarray*}
				\left| \frac{1}{n} \sum_{i=1}^{n} \|\bfx_i\|^2 \left\{ \left(\frac{\bfx_i}{\|\bfx_i\| }\right)^T \left( \vnh\vnh^T -VV^T \right) \frac{\bfx_i}{\|\bfx_i \| }\right\} \right|
				\leq \left| \lambda_{\max} \left(\vnh\vnh^T-VV^T \right)\right|  \frac{1}{n} \sum_{i=1}^{n} \|\bfx_i\|^2 .
			\end{eqnarray*}
			Note that $\|\bfx_i\|^2$s, $i=1,2,\ldots,n$, are independent random variables having mean $E\left( \bfx^{T}\bfx\right)=E(\tr \Sigma)=\tr(\Lambda_D)$ and finite variance by assumption (A). Therefore, by the Strong Law of Large Numbers (SLLN) (see, e.g., \cite{slln2008}) the following holds:
			\begin{eqnarray*}
				\frac{1}{n} \sum_{i=1}^{n} \|\bfx_i\|^2 \xrightarrow{a.s.} \tr(\Lambda_D)\Longrightarrow \frac{1}{n} \sum_{i=1}^{n} \|\bfx_i\|^2=O_{p}(1). 
			\end{eqnarray*}
			Further, recall that for a symmetric $A$, $\lambda_{\max}(A) \leq \sigma_{\max}(A)\leq \|A\|_{F}$, where $\sigma_{\max}(A)$ is the largest singular value of $A$. Thus 
			\begin{eqnarray*}
				\lambda_{\max} \left(\vnh\vnh^T-VV^T \right)\leq\left\|\vnh\vnh^T-VV^T \right\|_{F} 
				\leq \left\|\vnh\vnh^T-V\vnh^T \right\|_{F} +  \left\|V\vnh^T-VV^T \right\|_{F}.
			\end{eqnarray*}
			Now,
			\begin{eqnarray*}
				\sqrt{n} \left( \sigma \log n \right)^{-1} \left\| \left(\vnh^T-V \right) \vnh^T  \right\|_F = \sqrt{n} \left(\sigma \log n \right)^{-1} \left\| \vnh^T-V \right\|_F \xrightarrow{p}0,
			\end{eqnarray*}
			as shown in part {\bf i}.

			To see (\ref{eq_6}) note that $\bfx_i^T VV^T \bfx_i=w_{i}$, for $i=1,2,\ldots, n$, where $w_{i}$s are independent random variables having finite moments by (A). Therefore, by the Lideberg-Feller CLT (see \cite{Serfling_book}) the following holds:
			\begin{eqnarray*}
				\sqrt{n} \left(\sigma \log n \right)^{-1} \left| \frac{1}{n} \sum_{i=1}^{n} \|V^T \bfx_i \|^2 - E\left(\| V^T \bfx \|^2 \right) \right| &\xrightarrow{p}& 0.
			\end{eqnarray*} 
			Thus (\ref{eq_2}) is proved.
			
			\vskip10pt
			Next consider the convergence in (\ref{eq_3}). As before we split the proof in two parts:
			\begin{eqnarray}
				\sqrt{n} \left(\sigma \log n \right)^{-1}	\left\|\frac{1}{n} \sum_{i=1}^{n} \left(\vnh \vnh^T -VV^T \right) \bfx_i \right\|^2 &\xrightarrow{p}& 0, \label{eq_7} \\
				\sqrt{n} \left(\sigma \log n \right)^{-1}\left\|\frac{1}{n} \sum_{i=1}^{n} VV^T \bfx_i - E \left( VV^T \bfx \right) \right\|^2  &\xrightarrow{p}& {0}.\label{eq_8}
			\end{eqnarray}
			
			To see (\ref{eq_7}) observe as before that the LHS of (\ref{eq_7}) is less than or equal to
			\begin{eqnarray*}
				\frac{1}{n} \sum_{i=1}^{n}\left\| \left(\vnh \vnh^T -VV^T \right) \bfx_i \right\|^2
				\leq \left\| \vnh \vnh^T -VV^T  \right\|_F^2 \left( \frac{1}{n} \sum_{i=1}^{n}  \| \bfx_i\|^2 \right),
			\end{eqnarray*}
			as $\|A X \|^{2} \leq \lambda_{\max} (A^{T} A) \|X\|^2 =\sigma_{\max}^2 (A) \|X\|^2 \leq \|A\|_{F}^{2} \|x\|^2$ for any symmetric matrix $A$.
			
			As before, we write
			\begin{eqnarray*}
				\left\| \vnh \vnh^T -VV^T  \right\|_F^2 \leq
				\left\| \vnh \vnh^T - \vnh V^T  \right\|_F^2+ \left\| \vnh V^T -VV^T  \right\|_F^2. 	
			\end{eqnarray*}
			Thus following similar arguments, as provided earlier, one can show that (\ref{eq_7}) holds. 
			Next consider (\ref{eq_8}). Let ${ u}_i=VV^T \bfx_i$. 
			As $u_i$s are independent and third order moments of $u_i$s are finite, by Lindeberg Feller CLT, (\ref{eq_8}) is satisfied.  
			Finally, we show (\ref{eq_4}). Again, we split the proof in two main steps:
			\begin{eqnarray} 
				\sqrt{n} (\sigma\log n)^{-1}\left\| n^{-1} \sum_{i=1}^{n} \left( \|\vnh^T \bfx_i \|^2 \vnh \vnh^T \bfx_i - \|V^T \bfx_i \|^2 VV^T \bfx_i \right) \right\|  &\xrightarrow{p}& 0 \label{eq_9}\\
				\sqrt{n} (\sigma\log n)^{-1}\left\| n^{-1} \sum_{i=1}^{n}  \|V^T \bfx_i \|^2 VV^T \bfx_i - E\left( \|V^T \bfx \|^2 VV^T \bfx \right) \right\| &\xrightarrow{p}& 0 \label{eq_10}
			\end{eqnarray}
			To show (\ref{eq_9}), we split the problem into two parts
			\begin{eqnarray}
				\sqrt{n} (\sigma \log n)^{-1}\left\| n^{-1} \sum_{i=1}^{n} \left( \|\vnh^T \bfx_i \|^2 \vnh \vnh^T \bfx_i - \|\vnh^T \bfx_i \|^2 VV^T \bfx_i \right) \right\| &\xrightarrow{p}& 0 \label{eq_11}\\
				\sqrt{n} (\sigma \log n)^{-1}\left\| n^{-1} \sum_{i=1}^{n} \left( \|\vnh^T \bfx_i \|^2 VV^T  \bfx_i - \|V^T \bfx_i \|^2 VV^T \bfx_i \right) \right\| &\xrightarrow{p}& 0 \label{eq_12}
			\end{eqnarray}
			The proof of (\ref{eq_11}) is similar to that of (\ref{eq_7}), except here we have to show that $n^{-1}\sum_{i=1}^{n} \|\vnh^T\bfx_i \|^2 \|\bfx_i\|$ is bounded in probability. To show the boundedness, it is enough to show that $n^{-1}\sum_{i=1}^{n} \|\bfx_i\|^3$ converges (as $\bfx_{i}^{T}\vnh\vnh^T \bfx_{i}\leq \lambda_{\max} \left( \vnh \vnh^{T} \right)=\lambda_{\max} \left(  \vnh^{T} \vnh\right) =1$). Again, the random variables $\|\bfx_i\|^3$, $i=1,2,\ldots,n$, are independent and have finite third order moments by assumption (A). Therefore $n^{-1}\sum_{i=1}^{n} \|\bfx_i\|^3 \xrightarrow{a.s.} E \left\| X \right\|^{3} $, by SLLN, and hence is bounded in probability. 
			Finally, observe that the LHS of (\ref{eq_12}) is less than or equal to
			\begin{align}
				& n^{-1/2} (\sigma \log n)^{-1}\sum_{i=1}^{n} \left\| \left\{ \|\vnh^T \bfx_i \|^2 - \|V^T \bfx_i \|^2\right\} VV^T \bfx_i  \right\|\notag \\
				&= n^{-1/2} (\sigma \log n)^{-1}\sum_{i=1}^{n} \left| \bfx_i^{T} \left(\vnh \vnh^T -VV^T \right) \bfx_i \right|\left\|  V^T \bfx_i  \right\|\notag\\
				&\leq  n^{-1/2} (\sigma \log n)^{-1} \left|  \lambda_{\max}\left(\vnh \vnh^T -VV^T \right)  \right|  \sum_{i=1}^{n} \left\|   \bfx_i  \right\|^{3} \notag\\
				&\leq  n^{-1/2} (\sigma \log n)^{-1}   \sigma_{\max}\left(\vnh \vnh^T -VV^T \right)    \sum_{i=1}^{n} \left\|   \bfx_i  \right\|^{3} \notag \\
				&\leq  n^{-1/2} (\sigma \log n)^{-1}   \left\|\vnh \vnh^T -VV^T \right\|_{\mathrm{F}}    \sum_{i=1}^{n} \left\|   \bfx_i  \right\|^{3}. \label{eqn_E1}
			\end{align}
			As before, one can show that (\ref{eqn_E1}) converges in probability to zero.

			\vskip10pt
			
			\item {\it Remaining steps towards showing $\sqrt{n} (\sigma \log n)^{-1} \left(\cnh- \bfc \right) \xrightarrow{p} 0$.} Observe that,
			\begin{eqnarray*}
				\cnh-\bfc =\frac{1}{2} \left\{  \left( \widehat{H}^{+} - H^{+} \right) \widehat{\xi} + H^{+} \left(\widehat{\xi} -{\xi} \right) \right\}.
			\end{eqnarray*}
			Therefore,
			\begin{eqnarray*}
				2\sqrt{n} (\sigma \log n)^{-1}\|\cnh-\bfc\| \leq \sqrt{n} (\sigma \log n)^{-1}\| \widehat{H}^{+} - H^{+} \| \| \widehat{\xi} \| + \| H^{+}\|\sqrt{n} (\sigma \log n)^{-1} \|\widehat{ \xi} -{\xi} \|
			\end{eqnarray*}
			Thus, to show that $\sqrt{n} (\sigma \log n)^{-1} \left(\cnh- \bfc \right) \xrightarrow{p} 0$, it is enough to show that $\|\widehat{ \xi}\|$ is bounded in probability, and $\|H^{+}\|$ is bounded. From (ii) we observe that each component of $\widehat{ \xi}$ converges in probability to ${ \xi}$. Consider a number $N=N_{\epsilon}$ such that $P(\|\widehat{ \xi} -{\xi}\|>N/2)\leq \epsilon$, and $\|\xi\|<N/2$, then
			\begin{eqnarray*}
				P\left( \|\widehat{ \xi} \| >N \right) \leq P\left( \|\widehat{ \xi} -{\xi} \|+\|{ \xi}\| >N \right) \leq  P\left(\|\widehat{ \xi} -{ \xi} \| > \frac{N}{2} \right)  \leq \epsilon. 
			\end{eqnarray*}
			Finally, we have already seen that $H^{+}=V\Lambda^{-1}V^T$. Thus, 
			$$\|H^{+}\|\leq \|H^{+}\|_F=\sqrt{\tr(\Lambda^{-2})}\leq \lambda_{d+1}^{-1} \sqrt{d+1} = \sqrt{d+1}/\sigma^2,$$
			as $\lambda_{d+1}$ is bounded away from zero, the proof follows.
		\end{enumerate}
		
		\vskip15pt
		
		\item The next step is to show $\sqrt{n} (\sigma\log n)^{-1}(\rnh-r) \xrightarrow{p} 0$.
		Recall that
		$r=E\|\bfx - {c} \|$, and $\rnh=n^{-1}\sum_{i=1}^{n}  \| \bfx_i-\cnh \| $.
		We will prove this in the following two steps:
		\begin{eqnarray}
			\sqrt{n} (\sigma\log n)^{-1} \left| \frac{1}{n}\sum_{i=1}^{n} \left( \| \bfx_i-\cnh \| -\| \bfx_i-\bfc \| \right) \right| &\xrightarrow{p}& 0 \label{eq_13} \\
			\sqrt{n} (\sigma\log n)^{-1}\left| \frac{1}{n}\sum_{i=1}^{n}  \| \bfx_i-\bfc \| - E\| \bfx-\bfc \| \right| &\xrightarrow{p}& 0.  \label{eq_14}
		\end{eqnarray}
		We first prove (\ref{eq_13}). Observe that, by Jensen's inequality 
		$ \| \bfx_i-\cnh \| -\| \bfx_i-\bfc \| \leq \|\cnh-\bfc\| .$
		Thus, by the previous part (\ref{eq_13}) holds.
		Equation (\ref{eq_14}) can be shown by an application of Lideberg-Feller CLT, as the components are independent and identically distributed with finite second order moments.


		\item {\it Showing $\sqrt{n} (\sigma\log n)^{-1} \|v_i-\widehat{v}_i\| \xrightarrow{p} 0$, for $i=1,2,\ldots, d+1$.} This directly follows from the results of \cite{Davis_1977}.

	\end{enumerate}
\end{proof}

Before proving Corollary 5, we show the following two Lemmas first.
\begin{lemma}\label{lem:dH1}
	Given two hyper-spheres $S(c_1,r_1)$ and $S(c_2,r_2)$, their Hausdorff distance is
	$$d_H(S(c_1,r_1),S(c_2,r_2))=\|c_1-c_2\|+|r_1-r_2|.$$
\end{lemma}
\begin{proof}
	First assume $c_1=c_2=c$. Then by definition, for any $x\in S(c,r_1)$, the distance between $x$ and $S(c,r_2)$ is $$d(x,S(c,r_2))=|\|x-c\|-r_2|=|r_1-r_2|.$$
	Then we assume $c_1\neq c_2$. Observe that for any $x\in S(c,r_1)$, the distance between $x$ and $S(c_2,r_2)$ is 
	$$d(x,S(c_2,r_2))=|\|x-c_2\|-r_2|.$$
	Then observe that $$\underset{x\in S(c_1,r_1)}{\argmax}d(x,S(c_2,r_2))=c_1+r_1\frac{c_1-c_2}{\|c_1-c_2\|}.$$
	As a result,
	\begin{align*}
		&\sup_{x\in S(c_1,r_1)}\inf_{y\in S(c_2,r_2)} \|x-y\|=\sup_{x\in S(c_1,r_1)}|\|x-c_2\|-r_2|\\
		& = \left|\left\|c_1+r_1\frac{c_1-c_2}{\|c_1-c_2\|}-c_2\right\|-r_2\right| = |\|c_1-c_2\|+r_1-r_2|.
	\end{align*}
	By switching the indices, we have the Hausdorff distance:
	$$d_H(S(c_1,r_1),S(c_2,r_2))=\max\{|\|c_1-c_2\|+r_1-r_2|,|\|c_2-c_1\|+r_2-r_1|\}=\|c_1-c_2\|+|r_1-r_2|.$$
\end{proof}

\begin{lemma}\label{lem:dH2}
	Given two spheres 	$S(V_1,c_1,r_1)$ and $S(V_2,c_2,r_2)$, their Hausdorff distance is
	$$d_H(S(V_1,c_1,r_1),S(V_2,c_2,r_2))\leq\left(\|I-V_1V_1^\top\|+\|V_1V_1^ \top\|\right)\|c_1-c_2\|+|r_1-r_2|+r_1 \|V_1-V_2\|.$$
\end{lemma}	
\begin{proof}
	By triangular inequality,
	$$d_H(S(V_1,c_1,r_1),S(V_2,c_2,r_2))\leq d_H(S(V_1,c_1,r_1),S(V_1,c_2,r_2))+d_H(S(V_1,c_2,r_2),S(V_2,c_2,r_2))$$
	We start with the first term by assuming $V_1=V_2=V$ and observe that
	\begin{equation}\label{eqn:sdist2}
		d^2(x,S(V,c_2,r_2))=\|x-c_2-VV^\top(x-c_2)\|^2+(\|c_2+VV^\top(x-c_2)-c_2\|-r_2)^2.
	\end{equation}
	For any $x\in S(V,c_1,r_1)$, $x=c_1+VV^\top(x-c_1)$ so the first term in Equation \eqref{eqn:sdist2} becomes
	\begin{align*}
		&\|x-c_2-VV^\top(x-c_2)\|^2=\|x-c_2-VV^\top x+VV^\top c_2\|^2\\
		& \hspace{-0.53cm}= \|x-c_2-x+c_1-VV^\top c_1+VV^\top c_2\|^2 = \|c_1-c_2-VV^\top(c_1-c_2)\|^2,
	\end{align*}
	which does not depend on $x$. Let $\widehat c_1 = c_2+VV^\top(c_1-c_2)$ be the projection of $c_1$ to $c_2+V$. Then the above term can be simplied as $\|c_1-\widehat{c_1}\|^2$. Similarly, let $\widehat{x}=c_2+VV^\top(x-c_2)$ be the projection of $x$ to $c_2+V$, then the second term in Equation \eqref{eqn:sdist2} can be written as $(\|\widehat{x}-c_2\|-r_2)^2$.
	Recall that $x\in S(V,c_1,r_1)$, then $\widehat{x}\in S(V,\widehat{c_1},r_1)$ so by Lemma \ref{lem:dH1},  $$\underset{\widehat{x}\in S(V,\widehat{c_1},r_1)}{\argmax}(\|\widehat{x}-c_2\|-r_2)^2=\widehat{c_1}+r_1\frac{\widehat{c_1}-c_2}{\|\widehat{c_1}-c_2\|}$$
	if $\widehat{c_1}\neq c_2$, otherwise the $\argmax$ can be any point on $S(V,\widehat{c_1},r_1)$.
	The corresponding maximizer $x\in S(V,c_1,r_1)$ is determined by
	$$c_2+VV^\top(x-c_2)=\widehat{x}=\widehat{c_1}+r_1\frac{\widehat{c_1}-c_2}{\|\widehat{c_1}-c_2\|}.$$
	Recall that $VV^\top x = x-c_1+VV^\top c_1$, so we have
	\begin{align*}
		x &= c_1-VV^\top c_1+VV^\top x= c_1-VV^\top c_1-c_2+VV^\top c_2+\widehat{c_1}+r_1\frac{\widehat{c_1}-c_2}{\|\widehat{c_1}-c_2\|}\\
		& = c_1-\widehat{c_1}+\widehat{c_1}+r_1\frac{\widehat{c_1}-c_2}{\|\widehat{c_1}-c_2\|}= c_1+r_1\frac{\widehat{c_1}-c_2}{\|\widehat{c_1}-c_2\|}\\
		& = \underset{x\in S(V,c_1,r_1)}{\argmax}(\|\widehat{x}-c_2\|-r_2)^2
		= \underset{x\in S(V,c_1,r_1)}{\argmax}d(x,S(V,c_2,r_2)).
	\end{align*}
	As a result,
	\begin{align*}
		&\sup_{x\in S(V,c_1,r_1)}\inf_{y\in S(V,c_2,r_2)} = \sqrt{\|c_1-\widehat{c_1}\|^2+\left(\left\|\widehat{c_1}+r_1\frac{\widehat{c_1}-c_2}{\|\widehat{c_1}-c_2\|}-c_2\right\|-r_2\right)^2}\\
		& =  \sqrt{\|c_1-c_2-VV^\top(c_1-c_2)\|^2+\left(\left(1+\frac{r_1}{\|\widehat{c_1}-c_2\|}\right)\|\widehat{c_1}-c_2\|-r_2\right)^2}\\
		& = \sqrt{\|(I-VV^\top)(c_1-c_2)\|^2+\left(\|VV^\top(c_1-c_2)\|+r_1-r_2\right)^2}\\
		&\leq \|(I-VV^\top)(c_1-c_2)\|+\left|\|VV^\top(c_1-c_2)\|+r_1-r_2\right|\\
		&\leq \|I-VV^\top\|\|c_1-c_2\|+\|VV^ \top\|\|c_1-c_2\|+|r_1-r_2|,
	\end{align*}
	where $\|\cdot\|=\|\cdot\|_2$ is the induced $2$-norm for matrices. Note that this norm can also be replaced by the Frobenius norm since $\|A\|_2\leq \|A\|_F$. By symmetry, we conclude that
	$$d_H(S(V_1,c_1,r_1),S(V_1,c_2,r_2))\leq \|I-V_1V_1^\top\|\|c_1-c_2\|+\|V_1V_1^ \top\|\|c_1-c_2\|+|r_1-r_2|.$$

	Next, we assume $c_1=c_2=c$ and $r_1=r_2=r$. The Hausdorff distance is invariant under translation so we can assume $c=0$ without loss of generality.  Recalling that for any $x\in S(V_1,0,r)$ with $V_2V_2^\top x\neq 0$, then 
	\begin{align*}
		d^2(x,S(V_2,0,r))&=\left\|x-\frac{r}{\|V_2V_2^\top x\|}V_2V_2^\top x\right\|^2= r^2+r^2-2r x\cdot \frac{V_2V_2^\top x}{\|V_2V_2^\top x\|}\\
		& = 2r^2-2r^2 \frac{ x}{\|x\|}\cdot \frac{V_2V_2^\top x}{\|V_2V_2^\top x\|} = r^2\left(\left\|\frac{x}{\|x\|}-\frac{V_2V_2^\top x}{\|V_2V_2^\top x\|}\right\|^2\right).
	\end{align*}
	
	As a result, the Hausdorff distance is given by

	\begin{align*}
		d_H(S(V_1,c_1,r_1),S(V_2,c_1,r_1))&=r_1\min\left\{\min_{x\in V_1} \left\|\frac{x}{\|x\|}-\frac{V_2V_2^\top x}{\|V_2V_2^\top x\|}\right\|,\min_{y\in V_2} \left\|\frac{y}{\|y\|}-\frac{V_1V_1^\top x}{\|V_1V_1^\top y\|}\right\| \right\}
		\leq r_1 \|V_1-V_2\|.
	\end{align*}
	
	To conclude, the Hausdorff distance between two $d$-dimensional spheres is
	$$d_H(S(V_1,c_1,r_1),S(V_2,c_2,r_2))\leq\left(\|I-V_1V_1^\top\|+\|V_1V_1^ \top\|\right)\|c_1-c_2\|+|r_1-r_2|+r_1 L\|V_1-V_2\|.$$
	
\end{proof}

\begin{proof}[Proof of Corollary 5]
	The corollary is a direct consequence of Lemma \ref{lem:dH2} and Theorem 4.
\end{proof}

\section{Local SPCA}\label{sec:localSPCA}
In this section, we prove results related to local SPCA from Section 3.1. We first consider a single piece $M_k=M$ with radius $\alpha$ and sample size $n$. 

\begin{lemma}\label{lem:bias}
	Let $S_{V_y}(c_y,r_y)$ be the population solution of SPCA based on $y\sim \rho$ with $\supp(\rho)=M$ and $\mathrm{diam}(M)=\alpha$, then there exists $C>0$ such that
	\begin{equation}\label{eqn:mfderror}
		\sup_{x\in M} d(x,S_{V_y}(c_y,r_y))\leq C\alpha^2.
	\end{equation}
\end{lemma}

Before proving Lemma \ref{lem:bias}, we consider the following lemmas. Recall that a hyperplane can be viewed as a sphere with infinite radius, which is rigorously stated in the following lemma.
\begin{lemma}\label{inftyradius}
	Let $V$ be a $d$-dimensional subspace of $\RR^{d+1}$ and $\alpha>0$ is a fixed positive real number. Then for any $\epsilon$, there exists a sphere $S(c,r)$ such that $\displaystyle{\sup_{x\in V, \|x\|\leq \alpha}d(x,S(c,r))<\epsilon}$. 
\end{lemma}

\begin{proof}
	This is a direct corollary of Taylor expansion. In fact we can let $c=r\bf n$ where $\bf n$ is the unit normal vector of $V$ and $r=\frac{\alpha^2}{\epsilon}$.
\end{proof}
\begin{lemma}\label{PCAbias}
	Let $(\mu^*,V^*)$ be the best approximating hyperplane of $M$ (the population solution of PCA), then there exists $C>0$ such that
	$$\sup_{x\in M} d^2(x,\mu^*+V^*)\leq C\alpha^4.$$
\end{lemma}
\begin{proof}
	This is another direct consequence of Taylor expansion, see also Proposition \ref{lineartaylor2}.
\end{proof}

Now we prove Lemma \ref{lem:bias}. 
\begin{proof}[Proof of Lemma \ref{lem:bias}]
	Let $(c_y,r_y)$ be the solution of SPCA, then
	\begin{align*}
		(\|x-c_y\|-r_y)^2=\frac{(\|x-c_y\|^2-r_y^{2})^2}{(\|x-c_y\|+r_y)^2}
		\leq \frac{(\|x-c_y\|^2-r_y^{2})^2}{r_y^{2}}
		\leq \frac{(\|x-c_y\|^2-r_y^{2})^2}{\delta^2},
	\end{align*}
	{where $\delta$ is defined in Assumption (B) in Theorem 6}. As a result, it suffices to find the upper bound of $(\|x-c_y\|^2-r_y^{2})^2$ which is the loss function of SPCA. Lemma  \ref{PCAbias} implies that there exists an affine subspace $\mu+V$ and $C>0$ such that 
	$d^2(x,\mu+V)\leq C\alpha^4$ for any $x\in U$. Then set $\epsilon=C\alpha^2$ in Lemma \ref{inftyradius} so there exists $c,r$ such that $d\big(y,c+\frac{r}{\|y-c\|}(y-c)\big)\leq \epsilon=C\alpha^2$ for any $y=\mu+VV^\top x$ where $x\in U$. In fact, $r=\frac{\alpha^2}{\epsilon}=\frac{1}{C}$.
	For convenience, when $x$ is the original point in $U$, let $y$ be the linear projection of $x$ onto the affine subspace $\mu+V$ and $z$ be the spherical projection to sphere $S(c,r)$. By the triangular inequality, we have
	$$\|x-z\|\leq\|x-y\|+\|y-z\|\leq C\alpha^2+C\alpha^2=2C\alpha^2=C\alpha^2$$
	by Lemma \ref{inftyradius} and \ref{PCAbias}, where we are abusing $C$ for all constants without confusion. Then we evaluate the loss function at such $(c,r)$:
	\begin{align*}
		(\|x-c\|^2-r^2)^2&=(x^\top x-2c^\top x+c^\top c-r^2)^2\\
		&=((z+x-z)^\top (z+x-z)-2c^\top (z+x-z)+c^\top c-r^2)^2\\
		&=(0+\|x-z\|^2+2(z-c)^\top (x-z))^2\\
		&\leq (\|x-z\|^2+2|(z-c)^\top (x-z)|)^2\\
		&\leq (C^2\alpha^4+2\|z-c\|\|x-z\|)^2\\
		&\leq(C^2\alpha^4+2rC\alpha^2)^2\\
		&\sim 4r^2C^{2}\alpha^4=C\alpha^4
	\end{align*}
	when $\alpha$ is sufficiently small. To conclude,
	$$\sup_{x\in M} d^2(x,S_{V_y}(c_y,r_y))\leq\frac{\sup_{x\in M}(\|x-c_y\|^2-r_y^{2})^2}{\delta^2}\leq \frac{\sup_{x\in M}(\|x-c\|^2-r^2)^2}{\delta^2}\leq C\alpha^4.$$
\end{proof}
Then we consider the following Lemma regarding Hausdorff distance:
\begin{lemma}\label{lem:haus}
	For any $x\in\RR^D$ and two compact sets $A,B\subset \RR^D$, 
	$|d(x,A)- d(x,B)|\leq d_H(A,B).$
\end{lemma}
\begin{proof}
	Assume $x\in A$, then $d(x,A)=0$, so 
	$$|d(x,A)- d(x,B)|=d(x,B)\leq \sup_{x\in A}d(x,B)\leq d_H(A,B).$$
	A similar proof holds if $x\in B$. Then we assume $x\notin A\cup B$. Let $b_0\in B$ such that $d(x,B)=d(x,b_0)$ and $a_0\in A$ such that $d(b_0,A)=d(b_0,a_0)$, then
	$$d(x,A)- d(x,B)\leq d(x,a_0)-d(x,b_0)\leq d(a_0,b_0)=d(b_0,A)\leq d_H(A,B).$$
	Similarly we can show $d(x,B)- d(x,A)\leq d_H(A,B)$. 
\end{proof}
Now we are ready to prove Theorem 6.
\begin{proof}[Proof of Theorem 6]
	For the single piece case, let $\widehat{M} = S_{\widehat{V}_n}(\widehat{c}_n,\widehat{r}_n)$, where $\widehat{V}_n$, $\widehat{c}_n$ and $\widehat{r}_n$ are the solution of empirical SPCA on data $x_1,\cdots,x_n$. Similarly, let $V_z,~c_z,~r_z$ be the solution of population SPCA based on $z\sim \rho$ and $V_x,~c_x,~r_x$ be the solution of population SPCA based on $x\sim \rho*N(0,\sigma^2 I_D)$. Then we have
	\begin{align*}
		\sup_{y\in M}d(x,S_{\widehat{V}_n}(\widehat{c}_n,\widehat{r}_n))&\stackrel{Lemma~\ref{lem:haus}}{\leq} \sup_{y\in M}d(y,S_{V_z}(c_z,r_z))+d_H(S_{V_z}(c_z,r_z),S_{\widehat{V}_n}(\widehat{c}_n,\widehat{r}_n))\\
		&\stackrel{Lemma~\ref{lem:bias}}{\leq}C\alpha^2+d_H(S_{V_z}(c_z,r_z),S_{\widehat{V}_n}(\widehat{c}_n,\widehat{r}_n))\\&\stackrel{\text{tri. ineq.}}{\leq} C\alpha^2+d_H(S_{V_z}(c_z,r_z),S_{V_x}(c_x,r_x))+d_H(S_{V_x}(c_x,r_x),S_{\widehat{V}_n}(\widehat{c}_n,\widehat{r}_n))\\
		&\stackrel{Lemma~ \ref{lem:SPCAbias}}{\leq} C\alpha^2+C\sigma^2+d_H(S_{V_x}(c_x,r_x),S_{\widehat{V}_n}(\widehat{c}_n,\widehat{r}_n))\\
		&\stackrel{Lemma~ \ref{lem:SPCAvar}}{\leq} C\alpha^2+C\sigma^2+C \frac{\sigma \log n}{n^{1/2}}\\
		&\leq C\left(\alpha^2+\sigma^2+\frac{\sigma\log n}{n^{1/2}}\right).
	\end{align*}
	Then consider the general case, when $C_1,\cdots,C_K$ are partitions and $M_k\cap C_k$ is the submanifold of $M$ with radius $\alpha_k$ containing $n_k$ samples. Let $S_{\widehat{V}_n}^k(\widehat{c}_n,\widehat{r}_n)$ be the empirical solution of SPCA for the $k-$th partition. Then by the above argument, 
	$$\sup_{y\in M_k}d\left(y,S_{\widehat{V}_n}^k(\widehat{c}_n,\widehat{r}_n)\right)\leq C\left(\alpha_k^2+\sigma^2+\frac{\sigma\log n_k}{n_k^{1/2}}\right).$$
	Then by Assumption (C) and the bias-variance trade off, we know that
	$$\frac{\log n_k}{n_k^{1/2}}=\alpha_k^2 = \left(\frac{n_k}{n}\right)^{\frac{2}{d}}\Longrightarrow n_k =\mathrm{O}\left( n^{\frac{4}{d+4}}\right),~\alpha_k =\mathrm{O}\left( n^{-\frac{1}{d+4}}\right).$$
	As a result, $\sup_{y\in M_k}d(y,S_{\widehat{V}_n}^k(\widehat{c}_n,\widehat{r}_n))\leq C\left(\sigma^2+\frac{\sigma\log n}{n^{2/(d+4)}}\right)$. So we conclude that
	\begin{align*}
		\sup_{y\in M}d(y,\widehat{M})&=\sup_{k=1,\cdots,K}\sup_{y\in M_k}d(y,\widehat{M})\leq\sup_{k=1,\cdots,K}\sup_{y\in M_k}d(y,S_{\widehat{V}_n}^k(\widehat{c}_n,\widehat{r}_n)) \\
		&\leq \sup_{k=1,\cdots,K} C\left(\sigma^2+\frac{\sigma\log n}{n^{2/(d+4)}}\right) = C\left(\sigma^2+\frac{\sigma\log n}{n^{2/(d+4)}}\right).
	\end{align*}
\end{proof}

\section{Covering Number}\label{sec:covering}
In this section we prove the covering number theorem in Section 3.2. We split the proof into three cases: curves with $d=1, D=2$, surfaces with $d=2,D=3$, and the general case with any $d$ and $D$. The proof for curves is the simplest, while the proof for surfaces is different from the curve case and motivates the proof of the general case. 
\subsection{Curves (d=1)}
Throughout this section, $\gamma:[0,V]\rightarrow \RR^2$ is a $C^3$ curve with input $s$, the arc length parameter. Let $s_0\in[0,V]$ be fixed. Let $\kappa(s)$ be the curvature at point $\gamma(s)$.
\begin{proposition}\label{lineartaylor}
	Let $L(s)$ be the tangent line of $\gamma$ at point $\gamma(s_0)$, then $L$ is the unique line such that $d(\gamma(s),L)\leq \frac{K}{2}|s-s_0|^2$, so
	$$\lim_{s\rightarrow s_0} \frac{\|L(s)-\gamma(s)\|}{|s-s_0|}=0,$$
	where $\displaystyle{\kappa_{\max}=\sup_s |\kappa(s)|}$.
\end{proposition}
\begin{proof}
	This is a standard result so we will not prove it in this paper.
\end{proof}
\begin{corollary}\label{curvehypercovernum}
	Let $\epsilon>0$, then there exists $N\leq V(\frac{2\epsilon}{\kappa_{\max}})^{-\frac{1}{2}}$ tangent lines $L_1,\cdots,L_N$ at points $\gamma(s_1),\cdots,\gamma(s_n)$ such that $\forall p\in\gamma$, $\exists i$ s.t. $d(p,L_i)\leq \epsilon$.
\end{corollary}
\begin{proof}
	Let $s_1=0$, then by Proposition \ref{lineartaylor}, $\|\gamma(s)-L_1(s)\|< \frac{\kappa_{\max}}{2} |s|^2$, so if $s< (\frac{2\epsilon}{\kappa_{\max}})^{1/2}$, $\|\gamma(s)-L_1(s)\|<\epsilon$. Then starting from $\gamma\left(\left(\frac{2\epsilon}{\kappa_{\max}}\right)^{1/2}\right)$, we can approximate the next segment by another tangent line. We can repeat this process $N$ times to approximate each segment by a tangent line, so $N\leq V/(\frac{2\epsilon}{\kappa_{\max}})^{\frac{1}{2}}=V(\frac{2\epsilon}{\kappa_{\max}})^{-\frac{1}{2}}$.
\end{proof}
\begin{proposition}\label{quadtaylor}
	Let $C(s)$ be the osculating circle of $\gamma$ at point $\gamma(s_0)$. Then if $\kappa(s_0)\neq0$, $C$ is the unique circle such that $d(\gamma(s),C)\leq \frac{T+2\kappa_{\max}}{6}|s-s_0|^3$, so
	$$\lim_{s\rightarrow s_0} \frac{\|C(s)-\gamma(s)\|}{|s-s_0|^2}=0,$$
	where $\displaystyle{T=\sup_s|\gamma^{(3)}(s)|}$.
\end{proposition}
\begin{remark}
	Proposition \ref{quadtaylor} holds only when the osculating circle is non degenerate. 
	If the curvature of $\gamma$ at $\gamma(s_0)$ is $\kappa(s_0)=0$, the osculating circle $C$ degenerates to tangent line $L$. In this case, Proposition \ref{lineartaylor} applies.
\end{remark}
\begin{proof}
	As osculating circle is a local approximation of the curve, $|s-s_0|$ is assumed to be small in the following proof. The osculating circle $C$ has radius $r=\frac{1}{|\kappa(s_0)|}$ and center $\gamma(s_0)+\frac{1}{\kappa(s_0)}\mathbf{n}$, where $\mathbf{n}=\frac{\gamma''(s_0)}{\|\gamma''(s_0)\|}$ is the unit normal vector. Let $\{-\mathbf{n},t\}$ be the Frenet frame, where $t=\gamma'(s_0)$. Without loss of generality, assume $s_0=0$ and $\gamma(s_0)=0$. Under the Frenet frame, we can rewrite the osculating circle as 
	$$C(s)=\begin{bmatrix}
		-\frac{1}{\kappa}\\
		0
	\end{bmatrix}+\frac{1}{\kappa}\begin{bmatrix}
		\cos(\kappa s)\\
		\sin(\kappa s)
	\end{bmatrix}=\begin{bmatrix}
		r(-1+\cos(\kappa s))\\
		r\sin(\kappa s)
	\end{bmatrix}.$$
	The Taylor expansion for $\gamma$ can be written as
	\begin{align*}
		\gamma(s)&=\gamma(0)+\gamma'(0)s+\frac{1}{2}\gamma''(0)s^2+R_2(s)\\
		&=0+s\begin{bmatrix}
			0\\
			1
		\end{bmatrix}+
		\frac{s^2}{2}\begin{bmatrix}
			-\kappa\\
			0
		\end{bmatrix}+R_2(s)
		=\begin{bmatrix}
			-\frac{\kappa s^2}{2}\\
			s
		\end{bmatrix}+R_2(s),
	\end{align*}
	where $|R_2(s)|\leq \frac{T}{6} |s|^3$, so$\displaystyle{\lim_{s\rightarrow s_0}\frac{R_2(s)}{s^2}}=0$.
	As a result,
	\begin{align*}
		C(s)-\gamma(s)&=\begin{bmatrix}
			\frac{1}{\kappa}(-1+\cos(\kappa s))+\frac{\kappa s^2}{2}\\
			\frac{1}{\kappa}\sin(\kappa s)-s
		\end{bmatrix}-R_2(s)\\
		&=\begin{bmatrix}
			\frac{1}{\kappa}(-1+1-\frac{\kappa^2s^2}{2}+o(s^3))+\frac{\kappa s^2}{2}\\
			\frac{1}{\kappa}(\kappa s-\frac{1}{3}\kappa^3s^3+o(s^3))-s
		\end{bmatrix}-R_2(s)
		=\begin{bmatrix}
			o(s^3)\\
			-\frac{1}{3}\kappa^2s^3+o(s^3)
		\end{bmatrix}
		-R_2(s).\end{align*}
	As a result, 
	$$\|C(s)-\gamma(s)\|\leq  \frac{T+2|\kappa|}{6}|s|^3\leq\frac{T+2\kappa_{\max}}{6}|s|^3 .$$
	Now we prove the uniqueness. Observe the first entry of $C(s)-\gamma(s)$: 
	\begin{align*}
		&\frac{1}{\kappa'}\bigg(-1+1-\frac{\kappa^2s^2}{2}+o(s^3)\bigg)+\frac{\kappa s^2}{2}=o(s^3)\Longleftrightarrow \frac{\kappa^2s^2}{\kappa'}=\kappa s^2\Longleftrightarrow \kappa'=\kappa,
	\end{align*}
	which means $C(s)$ is the osculating circle.
\end{proof}
\begin{corollary}
	Let $\epsilon>0$, then there exists $N\leq V(\frac{6\epsilon}{T+2\kappa_{\max}})^{-\frac{1}{3}}$ osculating circles $C_1,\cdots,C_N$ at points $\gamma(s_1),\cdots,\gamma(s_n)$ such that $\forall p\in\gamma$, $\exists i$ s.t. $d(p,C_i)\leq \epsilon$.
\end{corollary}
\begin{proof}
	Let $s_1=0$, then by Proposition \ref{quadtaylor}, $\|\gamma(s)-C_1(s)\|< \frac{T+2\kappa_{\max}}{6} |s|^3$, so if $s< (\frac{6\epsilon}{T+2\kappa_{\max}})^{1/3}$, $\|\gamma(s)-C_1(s)\|<\epsilon$. Then starting from $(\frac{6\epsilon}{T+2\kappa_{\max}})^{1/3}$, repeat this process to find $N$ osculating circles, so $N\leq V/(\frac{6\epsilon}{T})^{\frac{1}{3}}=V(\frac{6\epsilon}{T+2\kappa_{\max}})^{-\frac{1}{3}}$.
\end{proof}
{Note that $k(s_0)\neq 0$ implies $\kappa_{\max}>0$, so $T+2\kappa_{\max}>0$ and $\frac{1}{T+2\kappa_{\max}}$ is well-defined.}

\subsection{Surfaces (d=2)}
Throughout this section, $M: U\rightarrow \RR^3$ is a regular $C^3$ surface parametrized by $x=x(u,v),y=y(u,v),z=z(u,v)$ where $U$ is a compact subset of $\RR^2$. 
Without loss of generality, assume $X_0=(x(0,0),y(0,0),z(0,0))=(0,0,0)\in M$ is fixed.
\begin{proposition}\label{lineartaylor2}
	Letting $H(u,v)$ be the tangent plane of $M$ at point $X_0$, $H$ is the unique plane such that $\|H(u,v)-M(u,v)\|^2\leq \frac{\kappa_{\max}}{2}\|(u,v)\|^2$, so
	$$\lim_{(u,v)\rightarrow (0,0)} \frac{\|H(u,v)-M(u,v)\|}{\|(u,v)\|}=0.$$
\end{proposition}
This is a higher dimensional analogue of Proposition \ref{lineartaylor}, but a similar analogue of Proposition \ref{quadtaylor} does not exist, which is a direct result from the following lemma.
\begin{lemma}
	There exists a sphere $S(u,v)$ such that 
	\begin{equation}\label{eqn:umbilical}
		\lim_{(u,v)\rightarrow (0,0)} \frac{\|S(u,v)-M(u,v)\|}{\|(u,v)\|^2}=0
	\end{equation}
	if and only if $X_0$ is an umbilical point.
\end{lemma}
\begin{proof}
	Recall that $X_0$ is an umbilical point $\Leftrightarrow$ $\frac{L}{E}=\frac{M}{F}=\frac{N}{G}=\alpha$, where \RNum{1}$=\begin{bmatrix}E & F\\ F& G\end{bmatrix}$ is the first fundamental form and \RNum{2}$=\begin{bmatrix}L & M\\ M& N\end{bmatrix}$ is the second fundamental form. Without loss of generality, assume $(u,v)$ are locally orthonormal parameters, which means $\<M_u,M_v\>=0$ and $\|M_u\|=\|M_v\|=1$ hence $E=G=1, F=0$. 
	
	If $X_0$ is an umbilical point, $L=N=\alpha, M=0$. Observe the Taylor expansion
	$$dX=X_0+M_udu+M_vdv+L du^2+Mdudv+N dv^2+o\left({\sqrt{u^2+v^2}}^2\right).$$
	Plugging in $L=N=\alpha$ and $M=0$, it is clear that there exists a sphere $S(u,v)$ such that 
	\eqref{eqn:umbilical} holds.
	If there exists a sphere $S(u,v)$ such that \eqref{eqn:umbilical} holds, then the quadratic terms in the Taylor expansion must satisfy $L=N$ and $M=0$, so $X_0$ is umbilical. 
\end{proof}

Although we can't find a sphere so that the error is third order in general, we can still find a sphere with this property in some direction, as shown in the following proposition.

\begin{proposition}
	Let $k_1\geq k_2$ be the principal curvatures of $M$ at $X_0$, $e_1, e_2$ be the corresponding principal directions and $\bf n$ be the normal vector at $X_0$. Then for any $k\in[k_2,k_1]$, let $S_k$ be the sphere centered at $c=X_0-\frac{1}{k}\bf n$ with radius $\frac{1}{|k|}$, there exists a curve $\gamma$ on $M$ and a constant $T$ such that 
	$$d(\gamma(s),S_k)\leq \frac{T}{6}|s|^3.$$ 
\end{proposition}
\begin{proof}
	Since $k\in[k_2,k_1]$, there exists a direction, represented by unit vector $\xi$ at $(0,0)\in U$ so that $k$ is the normal curvature in direction $\xi$. To be more specific, let $\gamma_\xi$ be the curve on $M$ in direction $\xi$; that is, $\gamma_\xi(s)=M(s\xi),s\in[L_1,L_2]$, where $L_1=\inf \{s|s\xi\in U \}$, $L_2=\inf \{s|s\xi\in U \}$. By the above construction, the curvature of $\gamma_\xi$ at $X_0$ is just $k$, so from Proposition \ref{quadtaylor}, we know that $$d(\gamma_\xi(s),C_\xi)\leq \frac{T}{6}|s|^3,$$
	where $T=\sup_s|\gamma_\xi^{(3)}(s)|$ and $C_\xi$ is a circle centered at $c=X_0-\frac{1}{k}\bf n$ with radius $\frac{1}{k}$. Since $C_\xi$ is just a great circle of $S_k$, we have the desired inequality:
	$$d(\gamma_\xi(s),S_k)\leq d(\gamma_\xi(s),C_\xi) \leq \frac{T}{6} |s|^3.$$
\end{proof}

\subsection{General Cases}
Throughout this section, $M$ is a d-dimensional $C^3$ compact manifold embedded in $\RR^D$. Let $p\in M$ be a fixed point and we can assume $p=0$ without loss of generality. Then we have the following proposition that is similar to Proposition \ref{lineartaylor} and Proposition \ref{lineartaylor2}.
\begin{proposition}\label{dhypertaylor}
	Let $T_{p}M$ be the tangent space of $M$ at $p$, then $d(x,T_{p}M)\leq \kappa_{\max}\|x\|^2$.
\end{proposition}

Before proving Theorem 8, we need a lemma regarding covering numbers and packing numbers of metric spaces; for more properties of these two numbers, see \cite{vershynin2018high}.
\begin{definition}
	Let $(X,d)$ be a metric space and $\delta>0$, then $\mathcal{N}\subset X$ is called a $\delta$-net if $\forall x\in X$, $\exists y\in \mathcal{N}:\ d(x,y)<\delta$. The covering number of $X$ denoted by $\mathcal{N}(X,d,\delta)$ is defined to be the smallest cardinality of an $\delta$-net of $(X,d)$. 
	
	$\mathcal{N}\subset X$ is said to be $\delta$-separated if  $d(x,y) > \delta$ for all distinct points $x,y \in\mathcal{N}$. The packing number of $X$ denoted by $\mathcal{P}(X,d,\epsilon)$ is defined to be the largest cardinality of an $\delta$-separated subset of $X$.
\end{definition}
\begin{lemma}\label{lemma:packing}
	Let $(M,g)$ be a compact $d$-dimensional Riemannian manifold embedded in $\RR^D$ and $d_g$ be the geodesic distance on $M$, then there exists constant $C=C(M)$ and $\delta>0$ such that $\forall r<\delta$, $\mathcal{N}(M,d_g,r)\leq CVr^{-d}$, where $V=\mathrm{Vol}_g(M)$ is the Riemannian volume of $M$. 
\end{lemma}
\begin{proof}
	
	First we claim that $\mathcal{N}(M,d_g,r)\leq \mathcal{P}(M,d_g,r)$. Let $\mathcal{N}=\{x_1,\cdots, x_N\}$ be an $r$-separated subset of $M$ whose cardinality is $N=\mathcal{P}(M,d_g,r)$, then for any $y\in M$, if $d_g(y,x_i)\geq r$ for all $i=1,\cdots,N$, then we can add $y$ to $\mathcal{N}$ to get another $r$-separated subset with cardinality $N+1$ which contradict the assumption. As a result, there exists $x_{i_0}\in\mathcal{N}$ such that $d_g(y,x_{i_0})< r$, which means $\mathcal{N}$ is a $r$-net of $M$ so the claim is true.
	
	Then we show that there exists $C=C(M)>0$ and $\delta=\delta(M)>0$ such that $\forall r\leq \delta$, $\mathcal{P}(M,d_g,r)\leq CVr^{-d}$. For any $x\in M$, there exists $\delta_x>0$ such that $\exp_{x}$ is homeomorphic on $B(0,\delta_x)\subset T_xM$. By compactness of $M$, there exists $\delta>0$ such that $\delta_x\geq \delta$ for any $x\in M$. Denote the Riemannian volume form $dV_g$ and the Lebesgue measure on $T_xM$ by $dV$, then
	$$\mathrm{Vol}_g(B_{d_g}(x,r))=\int_{B_{d_g}(x,r)}dV_M(y)=\int_{B(0,r)}|J_x(v)|dV(v),$$
	where $J_x(v)$ is the Jacobian of $\exp_x$ at $v$. By the compactness of $\{v\in T_xM\mid x\in M, \|v\|\leq \delta\}$, there exists constant $C$ such that $|J_x(v)|\geq C$ for any $x\in M$ and $ v\in T_xM$ with $\|v\|\leq \delta$. As a result,
	$$\mathrm{Vol}_{g}(B_{d_g}(x,r))\geq CC_dr^d=C(M)r^d,$$
	where $C_d$ is the volume of the $d$-dimensional unit ball. 
	
	Again, let $\mathcal{N}=\{x_1,\cdots, x_N\}$ be a $r$-separated subset of $M$ whose cardinality is $N=\mathcal{P}(M,d_g,r)$, then $\{B_{d_g}(x_i,r/2)\}_{i=1}^N$ are disjoint geodesic balls so
	$$V=\mathrm{Vol}_{g}(M)\geq \sum_{i=1}^N \mathrm{Vol}_g(B_{d_g}(x_i,r/2))\geq NC\left(\frac{r}{2}\right)^d,$$
	hence $N=\mathcal{P}(M,d_g,r)\leq \mathrm{Vol}_{g}/\left(Cr^d\right)=CVr^{-d}$, as desired. 
\end{proof}
Now we can prove Theorem 8.
\begin{proof}[Proof of Theorem 8]
	Firstly, we prove the first inequality in Equation (4).  First we focus on one local neighborhood of $p\in M$. Proposition \ref{dhypertaylor} shows that there exists a hyperplane $H$ such that when $\|x-p\|<\big(\frac{2\epsilon}{\kappa_{\max}}\big)^{\frac{1}{2}}=r$, $d(x,H)\leq \epsilon$. Since $\|x-p\|\leq d_g(x,p)$, for any $x\in B(p,r)$, $d(x,H)\leq \epsilon$. By Lemma \ref{lemma:packing}, there exists $C=C(M)$ and $\delta>0$ such that $\forall r\leq \delta$, $\mathcal{N}(M,d_g,r)\leq CVr^{-d}=CV\epsilon^{-\frac{d}{2}}$. As a result, when $\epsilon\leq \delta^2\kappa_{\max}/2$, the manifold $M$ can be covered by at most $CVr^{-d}$ geodesic balls and if we approximate each geodesic ball by a hyperplane, the approximation error is no more than $\epsilon$, which means $N_{\mathcal{H}}(\epsilon,M)\leq CV\epsilon^{-\frac{d}{2}}$, as desired.
	
	Then we prove the second inequality in Equantion (4) by considering two submanifolds $\overline{M}_\epsilon$ and $M-M_\epsilon$, which are both compact.
	
	\noindent	1. $M_\epsilon^\mathsf{c}$. Firstly we consider the worst part: the complement of $M_\epsilon$, which is compact as a closed subset of $M$. Let $\iota>0$ and let $M_\epsilon^\mathsf{c}(\iota)=\{x\in M: d(x,M_\epsilon^\mathsf{c})<\iota\}$. It is clear that $M_\epsilon^\mathsf{c}(\iota)$ is a open submanifold of $M$, hence is $C^3$. In addition, $M_\epsilon^\mathsf{c}(\iota)\supset M_\epsilon^\mathsf{c}$, so any cover of $M_\epsilon^\mathsf{c}(\iota)$ also covers $M_\epsilon^\mathsf{c}$.  We cover this subset of $M$ by geodesic balls with radius $\left(\frac{2\epsilon}{\kappa_{\max}}\right)^{\frac{1}{2}}$. The first part of the proof shows that this covering exists with approximation error no more than $\epsilon$, and the number of balls is less than or equal to $C\mathrm{Vol}(M_\epsilon^\mathsf{c}(\iota))\epsilon^{-\frac{d}{2}}$, where the constant $C$ is the same as the constant in the hyperplane case according to the proof of Lemma \ref{lemma:packing}. Then let $\iota\to0$, by the continuity of Riemannian volume, $\mathrm{Vol}(M_\epsilon^\mathsf{c}(\iota))\to V-V_\epsilon$, so the number of balls to cover $M_\epsilon^\mathsf{c}$ is no more than $C(V-V_\epsilon)\epsilon^{-\frac{d}{2}}$. 
	
	\noindent	2. $\overline{M}_\epsilon$. We cover this part by bigger geodesic balls. For any point $p\in F_\epsilon$, we have $\displaystyle{\sup_{v\in T^1_pM}\kappa(p,v)-\inf_{v\in T^1_pM}\kappa(p,v)\leq \left(\frac{2\epsilon}{\kappa_{\max}}\right)^{\frac{1}{2}}}$. Let $\displaystyle{k^*\in \left[\inf_{v\in T^1_pM}\kappa(p,v),\sup_{v\in T^1_pM}\kappa(p,v)\right]}$ be the curvature of a sphere to approximate $U\coloneqq B(p,(\frac{6\epsilon}{3+T})^\frac{1}{3})$. Then for any $q\in U$, if $q\in B(p,(\frac{2\epsilon}{\kappa_{\max}})^{\frac{1}{2}})$, when case 1 shows that the error is less than or equal to $\epsilon$, so we only need to consider $q\in U-B(p,(\frac{2\epsilon}{\kappa_{\max}})^{\frac{1}{2}})$, that is, $(\frac{2\epsilon}{\kappa_{\max}})^{\frac{1}{2}}\leq d(p,q)\leq (\frac{6\epsilon}{3+T})^\frac{1}{3}$. Let $\gamma_q(s)=\exp_p(s\log q)$ be the geodesic connecting $p$ and $q$, assume $\gamma_q$ at $p$ is $k_q$. Since both $k_q$ and $k^*$ are in $[k_d(p),k_1(p)]$, we have the following relation:
	$$|k_q-k^*|\leq \sup_{v\in T^1_pM}\kappa(p,v)-\inf_{v\in T^1_pM}\kappa(p,v)\leq \bigg(\frac{2\epsilon}{\kappa_{\max}}\bigg)^{\frac{1}{2}}.$$
	Recall in the proof of Proposition \ref{quadtaylor},  only the first three terms in the Taylor expansion of $\gamma_q$ matter; that is, $\gamma_q(0)$, $\gamma_q^{'}(s)$ and $\gamma_q^{''}(s)$, so we only need to consider the first two coordinates of $\gamma_q(s)$ and $C(s)$ while other coordinates are all $o(s^3)$. Similar to the proof of Proposition \ref{quadtaylor}, the first two coordinates of $C(s)-\gamma_q(s)$ are
	\begin{align*}
		&\begin{bmatrix}
			\frac{1}{k_q}(-1+\cos(k_q s))+\frac{k^* s^2}{2}\\
			\frac{1}{k_q}\sin(k_q s)-s
		\end{bmatrix}-R_2(s)\\
		&=\begin{bmatrix}
			\frac{1}{k_q}(-1+(1-\frac{k_q^2s^2}{2}+o(s^3))+\frac{k^* s^2}{2}\\
			\frac{1}{k_q}(k_q s-o(s^3))-s
		\end{bmatrix}-R_2(s)
		=\begin{bmatrix}
			\frac{s^2}{2}(k_q-k^*)+o(s^3)\\
			o(s^3)
		\end{bmatrix}
		-R_2(s),
	\end{align*}
	where $|R_2(s)|\leq\frac{T}{6}|s|^3$. We claim that 
	$\|C(s_q)-q\|\leq \epsilon$ where $s_q=d(p,q)$. Since 
	$|k_q-k^*|\leq \bigg(\frac{2\epsilon}{\kappa_{\max}}\bigg)^{\frac{1}{2}}\leq|s|\leq \bigg(\frac{6\epsilon}{3+T}\bigg)^{\frac{1}{3}},$
	$$\|C(s_q)-q\|\leq\frac{s^2}{2}|k_q-k^*|+\frac{T}{6}|s|^3\\
	\leq \bigg(\frac{1}{2}+\frac{T}{6}\bigg)|s|^3\leq \frac{3+T}{6}\frac{6\epsilon}{3+T}=\epsilon.
	$$
	As a result, we can cover $\overline{M}_\epsilon$ by geodesic balls with radius $r=\left(\frac{6\epsilon}{3+T}\right)^{\frac{1}{3}}$ so that the error is less than or equal to $\epsilon$. 
	On $\overline{M}_\epsilon$ the centers of geodesic balls are not arbitrary, but restricted to be in $F_\epsilon$. Define the smallest cardinality of an $r$-net in $F_\epsilon$ of $(\overline{M}_\epsilon,d_g)$ by $\mathcal{N}(\overline{M}_\epsilon,F_\epsilon,d_g,r)$, then we have $\mathcal{N}(\overline{M}_\epsilon,d_g,r)\leq\mathcal{N}(\overline{M}_\epsilon,F_\epsilon,d_g,r)$ because any $r$-net in $F_\epsilon$ is automatically an $r$-net in $\overline{M}_\epsilon$. In this situation, a similar claim in the proof of Lemma \ref{lemma:packing} does not hold; however, we claim that $\mathcal{N}(\overline{M}_\epsilon,F_\epsilon,d_g,r)\leq \mathcal{N}(\overline{M}_\epsilon,d_g,r/2)\leq CV_\epsilon \epsilon^{-\frac{d}{3}}$ for $r=\left(\frac{6\epsilon}{3+T}\right)^{\frac{1}{3}}$ and the theorem follows. To prove the claim, let $\mathcal{N}=\{x_1,\cdots, x_N\}\subset \overline{M}_\epsilon$ be an $r$-net of $\overline{M}_\epsilon$ whose cardinality is $N=\mathcal{N}(\overline{M}_\epsilon,d_g,r/2)$. Recall the definition of $M_\epsilon=\cup_{x\in F_\epsilon}B(x,r/2)$ so there exists $\{y_1,\cdots,y_N\}\subset F_\epsilon$ such that $d_g(x_i,y_i)\leq r/2$. Then for any $y\in \overline{M}_\epsilon$, there exists $x_{i_0}$ such that $d_g(y,x_{i_0})\leq r/2$. By triangle inequality, we have $d_g(y,y_{i_0})\leq d_g(y,x_{i_0})+d_g(x_i,y_{i_0})<r$. This implies $\{y_i\}_{i=1}^N$ is an $r$-net of $\overline{M}_\epsilon$ where $y_i\in F_\epsilon$ for $i=1,\cdots,N$, which means $\mathcal{N}(\overline{M}_\epsilon,F_\epsilon,d_g,r)\leq
	\mathcal{N}(\overline{M}_\epsilon,d_g,r/2)$.		
	As a result, the number of balls needed to cover $\overline{M}_\epsilon$ is less than or equal to $CV_{\epsilon}\epsilon^{-\frac{d}{3}}.$
	
	Based on the above two cases, the total number of balls
	$$N_S(\epsilon,M)\leq CV_{\epsilon}\epsilon^{-\frac{d}{3}}+C(V-V_{\epsilon})\epsilon^{-\frac{d}{2}}.$$
\end{proof}
\begin{remark}
	1. Corollary \ref{curvehypercovernum} is a special case of Theorem 8. \\
	2. When $d$ increases, the performance will be worse and worse, which is another representation of the curse of dimensionality.  Fortunately, $d$ is the intrinsic dimension of $M$, which is assumed to be small in most cases.\\		
\end{remark}
\begin{proof}[Proof of Proposition 10]
	It suffices to provide two manifolds with covering numbers achieving the upper bounds in Theorem 8.
	\begin{enumerate}
		\item Let $\gamma(t)=(t,t^2),\ t\in(0,1)$ so $\gamma^{''}(t)=(0,2)$ is constant. The covering number $N_\mathcal{H}(\epsilon,\gamma)$ follows.
		\item Let $\gamma(t)=(t,t^3),\ t\in (0,1)$ so $\gamma^{(3)}(t)=(0,6)$ is constant. The covering number $N_\mathcal{S}(\epsilon,\gamma)$ follows.
	\end{enumerate}
	
\end{proof}
\end{document}